\documentclass[conference]{IEEEtran}
\usepackage{amssymb}
\usepackage{graphicx} 
\usepackage{xcolor,colortbl}
\usepackage{epigraph}
\usepackage{mathtools}
\usepackage{amsmath}
\usepackage{cases} 
\usepackage{amsthm}
\usepackage[colorinlistoftodos]{todonotes}

\usepackage{times}
\usepackage{caption}
\usepackage{subfig}
\usepackage{booktabs}

\usepackage{multicol}
\usepackage[bookmarks=true]{hyperref}
\usepackage{cleveref}

\usepackage[linesnumbered,ruled,vlined]{algorithm2e}

\SetCommentSty{mycommfont}
\SetKwInput{KwInput}{Input}
\SetKwInput{KwOutput}{Output}
\newcommand{\spacestate}{\mathcal{S}}
\newcommand{\spacecontrol}{\mathcal{U}}
\newcommand{\diffusionTime}{r} 
\newcommand{\statetraj}{x_{0:T}}
\newcommand{\controltraj}{u_{0:T-1}}
\newcommand{\controltrajMean}{\bar{u}_{0:T-1}}
\newcommand{\controltrajTimeParticle}{u_{0:T}^{(r, i)}}

\newcommand{\phorizon}{T}
\newcommand{\nact}{n_a}

\newcommand{\measureP}{\mathbb{P}}

\newcommand{\measureU}{\mathbb{U}}

\newcommand{\measureUr}{{\mathbb{U}}^{(r)}}

\newcommand{\Brownian}{W}

\newcommand{\statecontrolpathcost}{J}

\newcommand{\laplaceSurrogateStateControlPathCostNext}{\xi(J\mid \measureU(\cdot\mid\controltrajMean+\deltaMean))}

\newcommand{\deltaMean}{\Delta(\controltrajMean)}

\newcommand{\samplesize}{N}

\newcommand{\cbomean}{\bar{u}_{0:T}}
\newcommand{\cbotemp}{\rho}

\newcommand{\softmaxappruJ}{\frac{\exp(-\rho\statecontrolpathcost (u_{0:T-1}^{(r,k)}\mid x_0))}{\sum_{j=1}^{\samplesize} \exp(-\rho\statecontrolpathcost (u_{0:T-1}^{(r,j)}\mid x_0))} }

\newtheorem{lemma}{Lemma}

\newtheorem{proposition}{Proposition}
\newtheorem{remark}{Remark}

\usepackage[most]{tcolorbox}
\tcbset{
  breakable,
  enhanced,
}
\newtcolorbox{questionbox}[1][]{
  colback=gray!10,
  colframe=black,
  fonttitle=\bfseries,
  title={Reviewer comments (#1)},
}

\newtcolorbox{answerbox}[1][]{
  colback=blue!5,
  colframe=blue!50!black,
  fonttitle=\bfseries,
  title={Response from authors (#1)},
}
\usepackage[numbers]{natbib}
\begin{document}

\title{Consensus-based Optimization (CBO):\\ Towards Global Optimality in Robotics}




%

\author{\authorblockN{Xudong Sun\authorrefmark{1}\authorrefmark{2}, Armand Jordana\authorrefmark{3}, Massimo Fornasier\authorrefmark{1}\authorrefmark{4}, Jalal Etesami\authorrefmark{1}\authorrefmark{2}, Majid Khadiv\authorrefmark{1}\authorrefmark{2}}
\authorblockA{\authorrefmark{1}School of Computation, Information and Technology, Technical University of Munich, Munich, Germany\\
Email: \{xudong.sun, massimo.fornasier, j.etesami, majid.khadiv\}@tum.de}
\authorblockA{\authorrefmark{2}Munich Institute of Robotics and Machine Intelligence (MIRMI), Munich, Germany}
\authorblockA{\authorrefmark{3}LAAS-CNRS, Université de Toulouse, CNRS, Toulouse, France\quad 
Email: armand.jordana@laas.fr}
\authorrefmark{4}Munich Center for Machine Learning (MCML), Munich, Germany}

\maketitle

\begin{abstract}
  Zero-order optimization has recently received significant attention for designing optimal trajectories and policies for robotic systems. However, most existing methods (e.g., MPPI, CEM, and CMA-ES) are local in nature, as they rely on gradient estimation. In this paper, we introduce consensus-based optimization (CBO) to robotics, which is guaranteed to converge to a global optimum under mild assumptions. We provide theoretical analysis and illustrative examples that give intuition into the fundamental differences between CBO and existing methods.
To demonstrate the scalability of CBO for robotics problems, we consider three challenging trajectory optimization scenarios: (1) a long-horizon problem for a simple system, (2) a dynamic balance problem for a highly underactuated system, and (3) a high-dimensional problem with only a terminal cost.
Our results show that CBO is able to achieve lower costs with respect to existing methods on all three challenging settings. This opens a new framework to study global trajectory optimization in robotics. 


\end{abstract}

\IEEEpeerreviewmaketitle

\section{Introduction}
State-of-the-art trajectory and policy optimization relies heavily on engineered heuristics, subtle design of cost functions, and careful weighting of different objectives~\cite{wensing2023optimization,ha2025learning}. One alternative to the cumbersome tuning of cost functions is to train policies directly retargeting human demonstrations to robots~\cite{allshire2025visual,yang2025omniretarget} or using large quantities of data collected via teleoperation~\cite{chi2025diffusion,intelligence2025pi_}. While these approaches have demonstrated impressive results in recent years, they can hardly guarantee any kind of optimality and might fail to obtain the full set of valid solutions to a given problem.
Hence, it is desirable to obtain favorable behaviors in robotics without reliance on human motion and other forms of high quality demonstration. Achieving this goal requires a capable optimization algorithm for trajectory optimization under high state-control dimensions and long horizons. 

Traditionally, trajectory optimization in robotics has relied on gradient-based methods~\cite{wensing2023optimization}. However, these techniques are well known to struggle with problems involving complex contact interactions, as (rigid) contact is inherently a hybrid phenomenon. Thanks to recent advances in parallelized simulation \cite{mujoco2023jax,makoviychuk2021isaac}, it is now possible to run massive parallel rollouts. This enables the simultaneous evaluation of many different control trajectories. As a result, zero-order (gradient-free) optimization has become increasingly popular for trajectory and policy optimization in robotics~\cite{jordana2025introduction,pan2025sampling}. By relying solely on function evaluations, these techniques avoid the challenging development of a differentiable simulator and are naturally contact-implicit.

Popular zero-order optimization techniques used for trajectory optimization in robotics are model predictive path integral (MPPI)~\cite{williams2016aggressive}, cross entropy method (CEM)~\cite{rubinstein2004cross}, and covariance matrix adaptation evolution strategy (CMA-ES)~\cite{hansen2001completely}. These methods have shown promise in solving complex robotic problems in real-time \cite{xue2025full,keshavarz2025control,pezzato2025sampling,crestaz2025td}.
The common ground among these techniques is that they iteratively perform a random search by generating samples and updating an estimate of the optimal solution based on the performance of each sample.
%
However, the sampling distribution is local, typically a multivariate Gaussian centered at the current estimate. The update step then approximates a gradient descent step~\cite{ollivier2017information}; this leads to local optimality.

In this paper, we explore the use of consensus-based optimization (CBO)~\cite{pinnau2017consensus} via particle dynamics to achieve global optimality in trajectory optimization for robotics. Instead of sampling locally, CBO uses a population of particles to explore globally. Each particle follows a diffusion dynamic that drives it towards a consensus point defined by the weighted average of the population, where the weights are determined by the trajectory costs. This approach allows for a more flexible exploration of the solution space and has been shown to converge to global optima under mild assumptions~\cite{carrillo2018analytical,fornasier2024consensus}.

Our major contributions in this work are:
\begin{itemize}
  \item We reinterpret zero-order optimization methods, widely used in robotics, under a general global optimization framework via surrogate objective minimization, which naturally highlights their weaknesses in achieving global optimality. We offer an intuitive explanation of how CBO works, how it addresses the weaknesses of other zero-order optimization methods, draw connections to current methods, and reinterpret other methods through the lens of the diffusion process dynamics. 
  \item We demonstrate that the CBO algorithm can identify meaningful global solutions in robotic trajectory optimization, even for high-dimensional humanoid models and problems with long time horizons. In particular, we illustrate how techniques based on local sampling can get stuck in local minima, regardless of the number of particles used, while CBO can find a global solution as the number of particles increases.
\end{itemize}

It is important to emphasize that CBO has been successfully tested on many global optimization benchmarks~\cite{JYZ} and applications \cite{fornasier2026}. This paper is the first to demonstrate its applicability to robotics problems and its superior performance compared to other zero-order optimization methods used in robotics.
It is also important to note that particle-based evolutionary strategies have a long history in robotics for performing global optimization~\cite{davidor1991genetic}. Rather than proposing a new optimization paradigm, here we introduce a unifying framework under which these meta-heuristic methods can be formally analyzed and studied within a rigorous mathematical setting.

The remainder of the paper is organized as follows. \Cref{sec:preliminaries} presents a unified framework for widely used zero-order optimization methods in robotics and highlights their inherently local nature. \Cref{sec:cbo} introduces the CBO formulation and provides intuition for its global behavior. The experimental results are presented in \Cref{sec:experiment}. Finally, \Cref{sec:conclusion} concludes the paper and summarizes our findings.

\section{Preliminaries}\label{sec:preliminaries}
\subsection{Notations}
In this section, we summarize the notions used throughout the paper to improve the readability:
  \begin{itemize}
    \item $r$: optimization iteration index
    \item $x_t\in\spacestate \subset \mathbb{R}^{n_s}$: state of system at time $t$ with state space $\spacestate$ of dimension $n_s$. 
    \item $u_{t}^{(r)}=(u_{t,1}^{(r)},\ldots,u_{t,n_a}^{(r)})\in \spacecontrol \subset \mathbb{R}^{\nact}$: control vector at time $t$ with $\nact$ number of actuations, at iteration $r$, where $\spacecontrol\subset \mathbb{R}^{\nact}$ is the control signal space.
    \item $\controltraj\in \prod_1^T \spacecontrol \subset \mathbb{R}^{\phorizon \times \nact }$: control signal (decision vector) corresponding to planning horizon $\phorizon$. 
    \item $\statetraj$: state trajectory  driven via control trajectory $\controltraj$ through system dynamic equation starting from state $x_0$. 
    \item Particle $i$: We use $\controltraj^{(r,i)}$ to denote the $i$th control signal at iteration $r$, which we call the $i$th particle. 
     \item $\measureUr(\controltraj)$: measure (a.k.a.~distribution) of control signal $\controltraj$ at iteration $r$.
    \item $\controltrajMean$: the location parameter of a distribution, i.e.~mean for parametric distribution or center (e.g.~target consensus point) of non-parametric distribution.

    \item $\cbotemp > 0$: temperature parameter controlling the selectiveness of the softmax weighting scheme.
  \end{itemize}



\subsection{The big picture of global optimization}\label{sec:global_opt}

Trajectory optimization finds a control sequence $\controltraj$ that minimizes a cost function $J$ given an initial state $x_0$: 
\begin{align}\label{eq:to_problem}
\min_{\controltraj} & \quad
    \statecontrolpathcost(\controltraj\mid x_0), 
\end{align}
where the cost is defined to be stagewise while satisfying the system's dynamics  $x_{t+1} = dyn(x_t, u_t)$:
\begin{align}
  \statecontrolpathcost(\controltraj\mid x_0) &=  \sum_{t=0}^{T-1} \ell_t(x_t, u_t) + \ell_T(x_T)\\
\text{s.t.} \quad &x_0=x_{init},\;\; x_{t+1} = dyn(x_t, u_t),
\end{align}
where $\{\ell_t\}$ and $\ell_T(x_T)$ denote the running and terminal costs, respectively. 

Achieving global optimality in trajectory optimization requires overcoming the limitations of traditional gradient-based optimization techniques that often get stuck in local minima of the objective function $J(\controltraj|x_0)$~\cite{wensing2023optimization}. One approach to mitigating this issue is to smooth the objective landscape by integrating the cost over a neighborhood of the current guess of the solution \cite{labbe2025aigo}. 

Formally, this smoothing can be defined given a probability distribution, $\measureU$, over the controls. More precisely, the smoothing surrogate of the objective function can be written as:
\begin{equation}
\xi(J\mid \measureU ) = \int J(\controltraj|x_0)d\measureU(\controltraj).\label{eq:surrogate}
\end{equation}
\begin{remark}
This smoothing can be interpreted as applying a low-pass filter to the cost landscape with a given bandwidth. 
A heavily smoothing filter may remove many local optima around the current solution and thus enable better solutions. However, it may also shift the global optimum of the surrogate landscape away from that of the original objective. Consequently, it is natural to begin the optimization with a strongly smoothing filter and gradually reduce its magnitude.
\end{remark}

%
Now, one may search for the probability distribution $\measureU$ that minimizes $\xi(J\mid \measureU)$. This naturally leads to an iterative optimization procedure in the spirit of gradient descent.
\begin{align}\label{eq:distribution_update_rule}
\measureU^{(r+1)}&= \measureU^{(r)} -\alpha\Delta \xi(J\mid \measureU^{(r)}),
\end{align}
where we use $\Delta$ to denote the variation operation to the distribution.
%

\begin{remark}
Typically, the distribution of controls $\measureU$ is chosen to be a multivariate Gaussian distribution. This is often referred to as randomized smoothing~\cite{Nesterov_Spokoiny_2017}. 
In this setting, we can write  $\measureU (\cdot \mid \controltrajMean)= \mathcal{N}(\controltrajMean, \Sigma)$, where $\controltrajMean$ is the mean of the gaussian distribution and $\Sigma$ is a fixed covariance matrix. In this case, the iterative procedure can be expressed as an iteration over the mean of the trajectory:
\begin{align}\label{eq:aigo_lebesgue}
\controltrajMean^{(r+1)}= \controltrajMean^{(r)}  -\alpha \nabla \xi(J\mid \measureU (\cdot \mid \controltrajMean^{(r)}))   
\end{align}
Note that the gradient is taken with respect to $\controltrajMean$.
The control mean can then be viewed as an estimate for the solution of the original minimization problem defined in~\eqref{eq:to_problem}.
\end{remark}


\begin{remark}
To enable global optimization, the following requirements need to be satisfied for representing $\measureU$:
\begin{itemize}
  \item For high-dimensional decision variables, the probability distribution $\measureU$ should be able to concentrate on the ``important'' regions of the solution space with favorable objective values and facilitate finite sample size approximation.
  \item The probability distribution $\measureU$ should be able to adapt and shrink its support to enable a refined search for the global optima.
\end{itemize}       
\end{remark}


\subsection{Zero-order methods used in robotics}\label{subsec:gradient_interp_sampling}
In this section, we cast the popular zero-order optimization methods used for trajectory optimization in robotics within the framework of~\eqref{eq:distribution_update_rule}. This perspective provides insight into their limitations and helps explain why our proposed method addresses them. We begin with the path integral framework.

\subsubsection{Path integral methods}
Path integral (PI) methods were originally introduced in~\cite{theodorou2010generalized} within the context of reinforcement learning. The method was derived from an information-theoretic perspective based on the Feynman-Kac lemma. Subsequently, a receding-horizon control variant, known as MPPI, was proposed in~\cite{williams2017model}.

The original derivation of PI updates happen in the state space; here, we are mostly interested in the control space solution directly. One can show that the control update can be computed as an expectation over sampled control trajectories weighted by their costs.
Given a current guess of the solution $\controltrajMean^{(r)}$, the algorithm samples $N$ random controls 
\begin{align}\label{eq:sampling_pro}
u_{0:T-1}^{(r,i)} \sim \mathcal{N}(\controltrajMean^{(r)}, \Sigma), 
\end{align}
where $i$ indexes the $N$ independent samples, and $\Sigma$ is a fixed covariance matrix.
These random control inputs are then used to compute the next guess according to the following  rule:
\begin{align}\label{eq:pi_update}
  \controltrajMean^{(r+1)}&=\sum_{i=1}^{\samplesize} w^{(r,i)} u_{0:T-1}^{(r,i)},\\
      \mbox{where} \quad   w^{(r,k)}    &= \softmaxappruJ , \label{eq:PI_weight}
\end{align}
where 
$w^{(r,k)}$ is referred to as the normalized weight computed from the cost of the $i^{th}$  trajectory.

The path integral update can be interpreted as a form of Gaussian smoothing over the objective function $J$ \cite{jordana2025introduction}.
To be more precise, one can define a surrogate of the form:
\begin{equation}
\xi(J\mid \measureU ) = -\frac{1}{\rho}\log \left(\int e^{-\rho J(\controltraj|x_0)} d\measureU(\controltraj)\right) .\label{eq:laplace_surrogate}
\end{equation}
Then, if we consider $\measureU (\cdot \mid \controltrajMean)\sim \mathcal{N}(\controltrajMean, \Sigma)$, the path integral updates in \eqref{eq:pi_update} can be recovered by applying \eqref{eq:distribution_update_rule}, with the gradient taken with respect to the mean $\controltrajMean$. The following proposition formalizes this result.




\begin{proposition}\label{prop:fisher}
The update for the mean of distribution $\controltrajMean$ under the PI framework in~\eqref{eq:pi_update} can be written as 
\begin{equation}
\controltrajMean^{(r+1)}= \controltrajMean^{(r)}-
\gamma 
\deltaMean,\label{eq:grad_fisher_kl_laplace}
\end{equation}
where $\deltaMean$ is the solution to the following constrained optimization problem with $\gamma$ being the Lagrange multiplier.
\begin{align}
&\arg\min_{\deltaMean} \laplaceSurrogateStateControlPathCostNext 
\label{eq:obj_natural_gradient_laplace_surrogate_plus_innerproduct_mppi}
\\
&s.t.~KL\Bigl(\measureU(u\mid \controltrajMean+\deltaMean) |\measureU(u\mid\controltrajMean)\Bigr) =\beta,\label{eq:constraint_natural_gradient_kl_sampling_distribution}
\end{align}
where $\beta$ defines the Kullback–Leibler (KL) divergence metric between the current distribution and new distribution.
\end{proposition}
We invite readers who are interested in the details to find proof in~\Cref{proof:fisher} in the supplementary material. 

\begin{remark}The Lagrange multiplier $\gamma=\frac{d\xi(J\mid \measureU;\controltrajMean+\Delta; \beta)}{d\beta}$ decides how much the optimal surrogate function will change per $\beta$ change. In practice, this constrained optimization problem is not solved explicitly; instead, the update in~\eqref{eq:pi_update} is applied directly, in which case the parameter $\beta$ could vary per iterations. 
\end{remark}

\begin{remark}[Particle dynamic under parameterized distribution updates]\label{remark:particle_dyn_mppi}
Consider two arbitrarily chosen particles (random control sequences) from iterations $r$ and $r+1$, the difference between them can be written as
\begin{align}
  \!\!u_{0:T-1}^{(r, i)}\!\!-u_{0:T-1}^{(r+1, j)} &\!= \controltrajMean^{(r)}\!+\!\Sigma^{\frac{1}{2}}\epsilon^{(r,i)}\!-\! \controltrajMean^{(r+1)} \!-\! \Sigma^{\frac{1}{2}}\epsilon^{(r+1,j)} \nonumber\\
                          \!&\!=  \Sigma^{\frac{1}{2}}(\epsilon^{(r,i)} \!-\!\epsilon^{(r+1,j)})\! +\!  \frac{1}{\gamma}\Delta( \controltrajMean),\label{eq:particle_dyn_mppi}
\end{align}
where $\{\epsilon^{(r,i)}\}$ are i.i.d. standard normal variables.
Thus, the potential improvements of $u^{(r+1, i)}$ over $u^{(r, j)}$ come from the major effect of $\Delta(\controltrajMean)$ plus a random exploration term $\Sigma^{\frac{1}{2}}(\epsilon^{(r,i)} -\epsilon^{(r+1,j)})$. Since the particles are reset after each iteration, there is no distinction between them (particles are "forgotten" after each iteration), resulting in  homogeneous  exploration. Any potential improvements beyond $\Delta(\controltrajMean)$ comes from repeatedly exploring the same mechanism of random directions.
\end{remark}

\begin{remark}
Note that when $\cbotemp$ is big enough in \eqref{eq:laplace_surrogate}:
\begin{align}
    \xi(\statecontrolpathcost\mid \measureU(\cdot\mid \controltrajMean))\approx \inf\statecontrolpathcost(\controltraj\mid x_0).
\end{align}
Thus, optimizing the surrogate is approximately equivalent to optimizing the lower bound, while the constraint penalizes large changes by restricting how much the distribution can shift at each iteration.
With a finite sample size, this involves choosing the best performing $\controltraj$ among the sampled population. 
\end{remark}

\begin{remark}[Curse of dimensionality in finite sample approximation of Gaussian expectation]\label{remark:mppi_lower_bound_opt_with_finite_sample_size}
When the decision variable lies in an extremely high-dimensional space, a finite sample size cannot realistically approximate the distribution;
thus, the lower bound of a finite sample does not approximate the mathematical formulation of~\eqref{eq:laplace_surrogate} faithfully. 
\end{remark}


\subsubsection{Covariance matrix adaptation}\label{sec:cma}
One way to mitigate the issue highlighted in~\Cref{remark:mppi_lower_bound_opt_with_finite_sample_size} is to adapt the covariance matrix of the sampling distribution. This allows the sampled population to concentrate along promising directions around the current solution while reducing exploration in less useful directions. This idea is known as Covariance Matrix Adaptation (CMA), and CMA-ES builds on it with a widely used evolution-path mechanism~\citep{hansen2001completely}.


For clarity, we use a simplified presentation of CMA-ES following the viewpoint of~\cite{jordana2025introduction}; the complete algorithm, including the full evolution-path machinery, is treated in~\cite{hansen2001completely}. 
Specifically, CMA-ES maintains a Gaussian search distribution over the control sequences and adapts its mean, covariance, and step size based on ranked samples~\cite{hansen2001completely}. At iteration $r$, offspring are sampled according to~\eqref{eq:cma_sample}, then evaluated through the cost function $J$:
\begin{align}
  u_{0:T-1}^{(r,i)} &= {\controltrajMean}^{(r)} + \alpha^{(r)} A^{(r)} \epsilon_i,\quad 
	  \epsilon_i\sim\mathcal{N}(0,I),\label{eq:cma_sample}
\end{align}
where $A$ is the decomposition of the covariance matrix $\Sigma$ as stated below:
\begin{equation}
 \Sigma^{(r)} = (\alpha^{(r)})^2 A^{(r)} (A^{(r)})^\top.    
\end{equation}
The mean is then updated only using an elite proportion of the current population and a moving average trick to stabilize the updates, which we formulate in~\eqref{eq:cma_mean_ada}. The covariance is adapted to capture successful search directions as~\eqref{eq:cma_cov_ada}, using $y^{(r,i)}$ in~\eqref{eq:cma_y}, which is the deviation of particle $i$ from the current population mean.
\begin{align}
\controltrajMean^{(r+1)} &= (1-\alpha^{(r)})\controltrajMean^{(r)} + \alpha^{(r)} \sum_{i=1}^{N_e} w^{(r,i)}\, u_{0:T-1}^{(r,i)}\label{eq:cma_mean_ada}\\
  \Sigma^{(r+1)} &= (1-\alpha^{(r)} )\Sigma^{(r)} + \alpha^{(r)} \sum_{i=1}^{N_e} w^{(r,i)}\, y^{(r,i)} {y^{(r,i)}}^\top \label{eq:cma_cov_ada}\\
  y^{(r,i)} &= u_{0:T-1}^{(r,i)}-{\controltrajMean}^{(r)}\label{eq:cma_y},
\end{align}
where $N_e$ is the number of elite particles. One choice of pertaining $w^{(r,i)}$, the weight for particle $i$, can be done via calculating~\eqref{eq:PI_weight}. 
The step size $\alpha^{(r)}$ is adjusted by a separate evolution path to control the overall search scale. 
This mechanism allows CMA-ES to adaptively stretch or shrink the sampling distribution along  promising directions.


As pointed out by \cite{akimoto2010bidirectional, ollivier2017information}, the update rule of CMA-ES (without the evolution path~\cite{hansen2001completely}) can be recovered by applying \eqref{eq:distribution_update_rule} with the gradient defined with respect to both the mean and covariance parameters.

\begin{remark}[Connection between CMA and CEM]\label{remark:cem_special_case_cma}
Setting $w^{(r,i)}=\frac{1}{N_e}$ with elite particles $N_e<N$ in~\eqref{eq:cma_cov_ada} gives the Cross Entropy Method (CEM)~\cite{rubinstein2004cross}.
\end{remark}

\begin{remark}[curse of distribution parametrization]\label{remark:curse_of_distribution_parameterization}
When a parameterized distribution is employed, its shape can be adjusted only through a limited number of parameters, which constrains the expressive flexibility of $\measureU(\controltraj)$.
The Gaussian distribution has additional innate geometric structures, such as symmetry, which typically do not fit control signal distributions. 
For instance, flipping the control signal with respect to an arbitrary center might result in a control signal that leads to undesirable behavior. In~\Cref{fig:gaussian_vs_irregular}, we illustrate the difference between a Gaussian distribution and an irregular, non-parametric distribution. The irregular distribution can develop longer, potentially asymmetric tails along important directions, whereas a shrinking Gaussian is much less flexible in accommodating such behavior. In~\Cref{sec:cbo}, we introduce CBO and explain how an irregular, non-parameteric distribution can be achieved, which also has the ability to concentrate on important regions.
\end{remark}

\begin{figure}[h!]
\centering\includegraphics[width=0.5\linewidth]{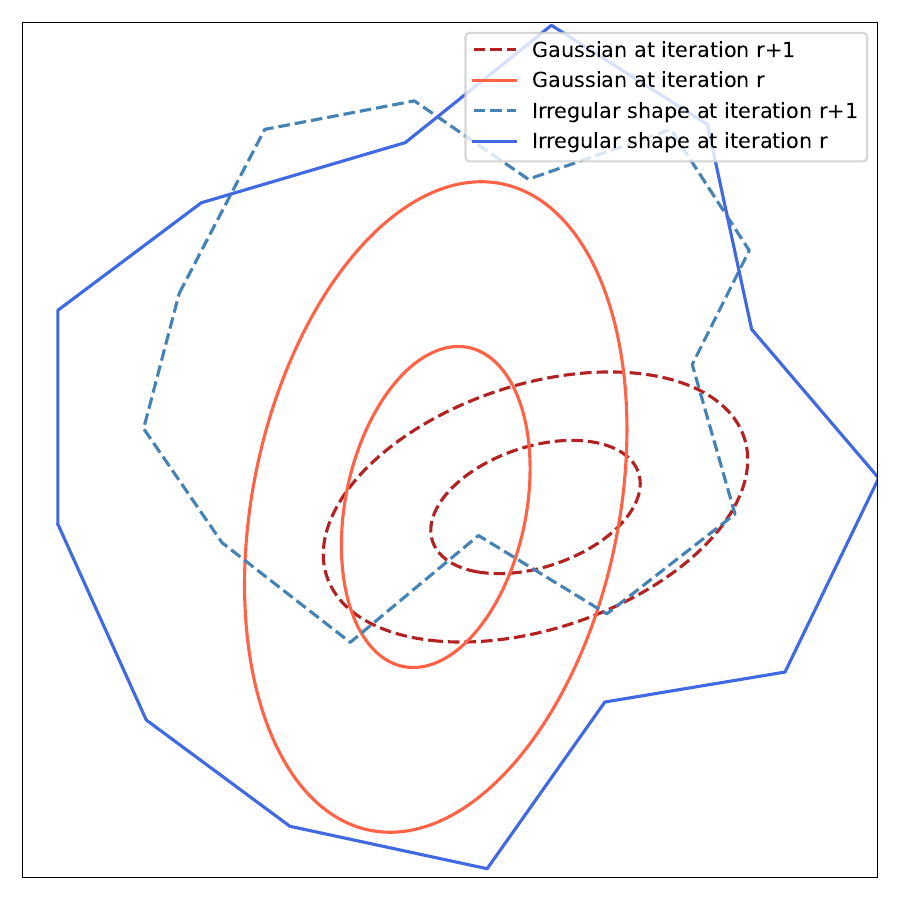}
    \vspace{-2mm}
    \caption{An irregular, non-parametric distribution can focus probability on important directions by having longer tails, whereas a shrinking Gaussian is constrained by its fixed symmetric shape and has less flexibility.}~\label{fig:gaussian_vs_irregular}
    \vspace{-5mm}
\end{figure}




\section{Consensus based optimization (CBO)}\label{sec:cbo}

To address the challenges of parameterized sampling-based optimization discussed in~\Cref{sec:preliminaries}, we introduce consensus-based optimization (CBO) to the robotics community.
Instead of explicitly parameterizing the distribution $\measureU^{(r)}$ at iteration~$r$, CBO uses a population of particles ${\{\controltrajTimeParticle\}}$ 
to approximate the distribution $\measureU^{(r)}$, where $i$ indexes the particles of the population.
Under mild assumptions, CBO is proven to converge to the global optimum~\cite{carrillo2018analytical, fornasier2024consensus}.

\subsection{Basics of CBO}
In CBO, each particle follows the following stochastic differential equation (SDE): 
\begin{equation}
d\controltrajTimeParticle= -\lambda\left(\controltrajTimeParticle-\cbomean^{(\diffusionTime)}\right)d\diffusionTime + \sigma {\left|\controltrajTimeParticle - \cbomean^{(\diffusionTime)}\right|}d\Brownian^{(\diffusionTime, i)},\label{eq:cbo_dyn}
\end{equation}
where $\diffusionTime$ denotes the iteration index, 
$\lambda$ is the decay rate 
governing how fast particle $i$ drifts toward the consensus point $\cbomean^{(\diffusionTime)}$ with drift rate $-\lambda(\controltrajTimeParticle-\cbomean^{(\diffusionTime)})$. The consensus point is defined by
%
\begin{equation}
  \cbomean^{(\diffusionTime)}=\sum_i w^{(\diffusionTime, i)}u_{0:\phorizon}^{(\diffusionTime, i)},\label{eq:cbo_mean_particle_dirac}
\end{equation}
where $w^{(r,i)}$ denotes the weight of particle $i$ at iteration $\diffusionTime$ defined as in~\eqref{eq:PI_weight}.
%
%
$\sigma\in\mathbb{R}$ in~\eqref{eq:cbo_dyn} controls the intensity of exploration via Brownian motion $\Brownian^{(\diffusionTime, i)}$. In other words, larger $\sigma$ leads to more exploration. 
In discrete time iteration,~\eqref{eq:cbo_dyn} becomes
\begin{align}
 \Delta u^{(r,i)}= \controltraj^{(r+1, i)}- \controltraj^{(r, i)} &=-\lambda \left(\controltraj^{(r, i)}-\cbomean^{(r)}\right)\Delta r \nonumber \\
  & + \sigma \sqrt{\Delta r} {\left|\controltraj^{(r, i)} - \cbomean^{(r)}\right|}\Delta \Brownian^{(r, i)}\label{eq:cbo_dyn_discrete}
\end{align}
where $\Delta \Brownian^{(r, i)}$ is sampled from standard normal distribution $\mathcal{N}(0, I)$. 
Algorithm~\ref{algo:CBO} summarizes the steps of CBO, where the particle-wise for loops can be parallelized on a GPU. This parallel structure keeps CBO computationally fast in practice: in the experiment section, we demonstrate that empirically its per-iteration runtime is comparable to CMA-ES and MPPI, as shown in~\Cref{tab:runtime-comparison} in~\Cref{subsec:exp_d_cartpole}.

\begin{algorithm}[h!]
\DontPrintSemicolon
\KwInput{Initial state $x_0 \in \mathcal{S}$, SDE integration time $\Delta r$, initial population $\{\controltraj^{(0,i)}\}$ of size $N$. $\sigma,\lambda$}
\While{stopping criterion is not met, }{
\For{$i=1,\ldots, N$}{Evaluate $J$ value for particle $i$}      
Calculate target consensus point $\cbomean$ with~\eqref{eq:cbo_mean_particle_dirac}\; 
\For{$i=1,\ldots,N$}{
  Calculate drift and exploration term in~\eqref{eq:cbo_dyn_discrete} to update particle\;   
}
}
\KwOutput{final population $\{\controltraj^{(r^{*},i)}\}$ at iteration $r^{*}$}
\caption{CBO for trajectory optimization}\label{algo:CBO}
\end{algorithm}

\begin{remark}[]\label{remark:cbo_wasserstein_dyn}
In~\Cref{remark:particle_dyn_mppi}, we derived the particle dynamic for the PI update that reads as 
$u^{(r, i)}-u^{(r+1, j)} =  \Sigma^{\frac{1}{2}}(\epsilon^{(r,i)} -\epsilon^{(r+1,j)}) +  \frac{1}{\gamma}\Delta( \controltrajMean)$ and  discussed the homogeneous drift behavior and exploration. In contrast, CBO dynamic in~\eqref{eq:cbo_dyn} indicates that different particles exhibit different behaviors depending on how far away the particle is from the target consensus point. The exploration term is also particle dependent, again depending on how far away the particle is from the target consensus point; particles farther away explore more.
\end{remark}

\begin{remark}[$\lambda$ and $\Delta r$]
In~\eqref{eq:cbo_dyn_discrete}, $\Delta r$ corresponds to Euler-Maruyama integration time for the stochastic differential equations. Numerically, $\Delta r$ can be absorbed into $\lambda$ as a lumped up parameter. 
  When $\lambda \Delta r=1$, $\sigma= \sqrt{\Delta r}$, the CBO dynamic in~\eqref{eq:cbo_dyn_discrete} reduces to 
\begin{align}
  \controltraj^{(r+1, i)}=&\cbomean^{(r)} + \sigma^2 {\left|\controltraj^{(r, i)} - \cbomean^{(r)}\right|}\Delta \Brownian^{(r, i)},\label{eq:cbo_dyn_discrete_reduce}
\end{align}
which indicates that at each iteration, each particle is resampled around the target consensus point with a variance proportional to its distance to the consensus point. This corresponds to a special case of CBO called consensus hopping~\cite{riedl2023gradient}. 

Note that the specific type of consensus hopping we defined in~\eqref{eq:cbo_dyn_discrete_reduce} is different from the sample generation in the PI update in~\eqref{eq:sampling_pro}, 
where each particle at the current generation (iteration)  is resampled around the target consensus point 
with a fixed covariance $\Sigma$, independent of the particle's distance to the consensus point.
\end{remark}

\subsection{Advantages of CBO in global optimization}
We use the example of~\Cref{fig:finite_locals_unique_global} as an illustration of the advantages of CBO where we constructed a cost function landscape in a contour plot with a unique global optimizer (marked as blue star) and several local optimizers (marked as orange disks). As a \textbf{hypothetical setting} for the purpose of illustration only, we assume at the current iteration, the target consensus is marked with a red cross in the figure. The following key observations explain the advantage of CBO over other sampling-based methods.
\begin{figure}[htpb]
  \centering
  \includegraphics[width=0.99\linewidth]{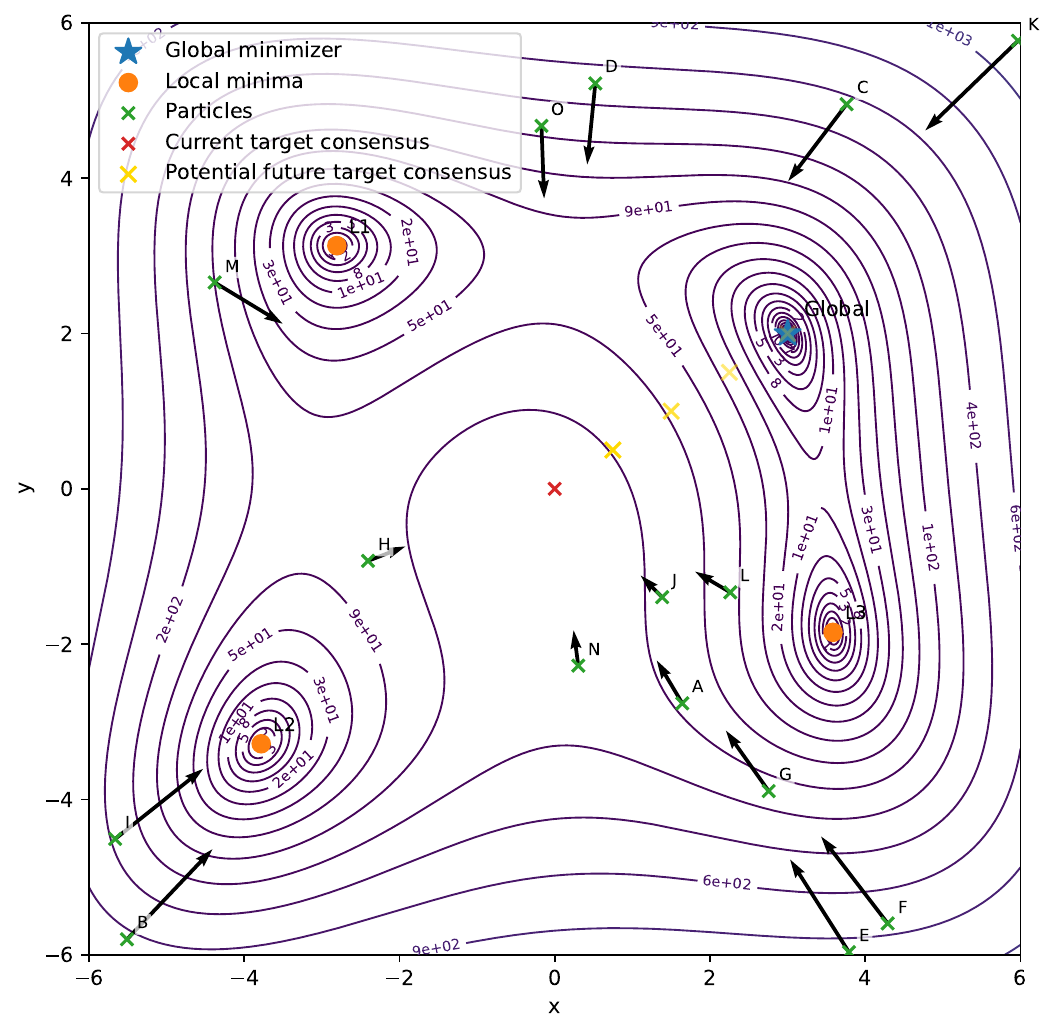}
  \vspace{-7mm}
  \caption{A cost function landscape in a contour plot constructed with a unique global optimizer (marked star) and several local optimizers for illustration of how CBO works and its advantages. See text for more explanation.
  }\label{fig:finite_locals_unique_global}
\end{figure}

\begin{itemize}
\item The dynamic of each particle is not directly driven by the gradient of the objective function, but by a drift term that “drags” or “attracts” the particle toward the target consensus point at the current iteration in~\eqref{eq:cbo_dyn}. As a result, the particle dynamics are agnostic to local optima of the objective function. For instance, in the illustrative example of~\Cref{fig:finite_locals_unique_global}, particle $B$ at the lower-left corner can travel through the local minimum near $(-4,-3)$ since its dynamic aims to drive it toward the goal position of the target consensus point, instead of the local minimum.
\item The inertia of each particle enables local exploration around the particle position from the previous iteration. In contrast, sampling-based methods discard particle histories (see~\eqref{eq:particle_dyn_mppi}). For instance, in the illustrative~\Cref{fig:finite_locals_unique_global}, the drift vectors (arrows in the plot) for particle $K,C,O,D$ determines a dominating direction towards the current consensus point, however, the exploration term in~\Cref{eq:cbo_dyn} adds noise on top, \textbf{proportional to the distance to the current consensus point} (Particles sitting on the consensus point has no exploration, a necessary condition for convergence. Particles far way get explored more to search for better solution), avoiding a straight-line behavior and allowing the particles to explore their neighborhood. This gives them a good chance to approach the vicinity of the global optimizer. 
\item As particles move toward the target consensus point, each particle explores along different directions. Along these drifting paths, once a better solution is found, the target consensus point shifts toward that solution, and the drift terms of all particles change accordingly, inducing a new set of exploration directions. This results in a more diverse exploration mechanism than the homogeneous exploration used in the PI updates (see~\eqref{eq:particle_dyn_mppi}). For instance, as particle $K$ and $C$ in~\Cref{fig:finite_locals_unique_global} being pulled toward the current target consensus point $(0,0)$, they have a good chance to find better position with lower cost than the current consensus point since they are near the global optimizer. 
\item The non-parametric nature of CBO allows it to adapt to complex cost landscapes more effectively than parameterized distributions, which may struggle to capture intricate solution spaces. The particles gradually gather in the vicinity of the target consensus point, analogous to the covariance matrix shrinkage in CMA. Unlike CMA, however, the empirical distribution of CBO is not Gaussian-shaped, but instead exhibits an irregular shape as illustrated in Figure \ref{fig:gaussian_vs_irregular}.
\end{itemize}

\begin{remark}[How CBO alleviates the curse of dimensionality and the curse of distribution parameterization]\label{remark:cbo_lower_bound_opt_with_finite_sample_size}
In~\Cref{remark:mppi_lower_bound_opt_with_finite_sample_size}, we pointed out that when the decision variables have high dimensions, faithfully approximating the cost surrogate in~\eqref{eq:laplace_surrogate} becomes unrealistic.  
CMA tries to resolve this issue by adapting its covariance matrix to let finite samples better concentrate on favorable directions, with potential shrunken eigenvalues, but is still restricted by the parameterized form of the as we noted in~\Cref{remark:curse_of_distribution_parameterization}.
Instead, the population for CBO is evolving via the dynamic~\eqref{eq:cbo_dyn}, which offers a more flexible adaptive capacity to capture interesting regions of the decision variable space. In~\Cref{subsec:cbo_properties}, we demonstrate such an irregular, non-parametric distribution of solutions from CBO is effective in achieving global optimality.

\end{remark}

\subsection{Properties of CBO}\label{subsec:cbo_properties}
First we show CBO falls into the global optimization framework of~\eqref{eq:distribution_update_rule} with the surrogate function in~\eqref{eq:surrogate}
under the assumption that $J(u\mid x_0)$ satisfies the so-called inverse continuity property and relaxed Lipschitz continuous property, which we state below:
\begin{align}
  \frac{1}{L(u)}\|J(u\mid x_0)-J(u^{*}\mid x_0)\| &\le\|u-u^{*}\| \label{eq:lipschitz}\\
                                        &\le \frac{1}{\eta}\|J(u\mid x_0)-J(u^{*}\mid x_0)\|\label{eq:assumption_inverse_continuity} \\
                                        &\forall u\in B_{\kappa}(u^{*})\label{eq:cbo_ball}
\end{align} 
where $B_{\kappa}(u^{*})$ is a neighborhood of the optima $u^{*}$ with radius $\kappa$, and $\eta>0$ is a constant.

Formally, to demonstrate the irregular, non-parametric distribution of particles of CBO is effective in achieving global optimality, we have the following results:
\begin{proposition}\label{prop_cbo}
Without loss of generality, assume $J(u^{*} \mid x_0)=0$.
When~\Crefrange{eq:lipschitz}{eq:cbo_ball} hold, there exist $ \rho, r^*>0$ large enough, 
where 
\begin{equation}
  r^{*}\propto \frac{1}{2\lambda-n_a\times T \sigma^2}, \label{eq:iteration_steps_decay_selection}
\end{equation}
such that, with high probability, CBO improves the surrogate function following the formulation in~\eqref{eq:distribution_update_rule} after at most $r^{*}$ iterations, i.e., we have
  \begin{align}
  &\int J(\controltraj|x_0)d\measureU^{(r^{*})}(\controltraj\mid \controltrajMean^{(r^{*})})\nonumber\\  < &\int J(\controltraj|x_0)d\measureU^{(0)}(\controltraj\mid \controltrajMean^{(0)}).
  \end{align}

In addition, due to the structure of~$\measureU$ evolution driven by the CBO dynamic in~\eqref{eq:cbo_dyn},
the target consensus point converges to the global optimizer with high probability.

\end{proposition}
Proof is in~\Cref{proof:cbo} of the supplementary material.  

\begin{remark}[choice of decay rate $\lambda$]\label{remark:choice_lambda_cbo}
  Equation~\eqref{eq:iteration_steps_decay_selection} gives us a convenient reference of selection of decay rate $\lambda$ with respect to the noise level $\sigma^2$ to ensure convergence of the dynamics of~\eqref{eq:cbo_dyn_discrete}. Concretely, the convergence can be ensured by $2\lambda > n_a \times T \times \sigma^2 $. For hyperparameter tuning, we recommend simply calculate the lowest $\lambda$ that satisfies this formula and then try a grid of values, while ensuring $\lambda \Delta r<1$. Increasing $\lambda$ improves convergence rate but too large $\lambda$ breaks the performance, see an empirical comparision of different $\lambda$ in~\Cref{fig:sweep_lambda} in the appendix to understand this tradeoff.
\end{remark}


\section{Experiment}\label{sec:experiment}
In this section, we show how the CBO algorithm can enable finding better solutions on three challenging robotics problems, highlighting different challenges in trajectory optimization for robotics. In what follows, first, we explain the general setting of our experiments and then present the results.

\subsection{Experiment setting}
To ensure fair comparisons, we use the same environment and the same initial particle population for all methods compared within each experiment whenever possible. In addition, we align shared hyperparameters across algorithms, for example, by using the same temperature $\rho$ and matching the $\sigma$  parameter of CBO with the exploration terms of the other methods.
We choose $\lambda$ in~\eqref{eq:cbo_dyn} according to~\Cref{remark:choice_lambda_cbo}.

For~\Cref{subsec:exp_d_cartpole} and~\Cref{subsec:exp_g1}, our implementation is based on \textit{Hydrax}~\cite{kurtz2024hydrax} and we use the implementation of CMA-ES from \textit{Evosax}~\cite{lange2023evosax}.  
Our implementation of the algorithms and experiments can be found in \url{https://github.com/Atarilab/cbo_to}.
\subsection{Long horizon planning}\label{subsec:exp_jungle}

In this example, we study a trajectory optimization problem involving a simplified dynamical system with a low-dimensional state space but an extremely long planning horizon. As illustrated in~\Cref{fig:jungle}, the agent is modeled as a point mass but represented by a chassis for better visualization (shown as a small blue rectangle) that travels from an initial position (black square located at the center left) towards a goal position (marked by a small red disk) through a confined space (depicted as a large green square). The tunnel has an upward-facing opening and is bounded by a wall on the left side (shown as black slats) that extends upward, as well as a shorter wall on the right side.
We deliberately design the left wall to extend upward so that reaching it corresponds to a local minimum, since the agent cannot penetrate the wall.
In many robotics control problems, feasible control signals lie in very “narrow” regions. To emulate this phenomenon, we randomly place obstacles (shown as dark green disks) such that only carefully chosen control signals allow the agent to successfully navigate through the obstacles and reach the tunnel.

%

We use the following simple first order dynamic model without inertia (i.e., velocity can be changed instantaneously):
\begin{align}
q_x(t+1) = q_x(t) + v \cos(\theta) \delta t,\\
q_y(t+1) = q_y(t) + v\sin(\theta) \delta t.
\end{align}
To avoid optimization in the radian space, instead of optimizing the sequence of $\theta \in[0,2\pi], v\in \mathbb{R}$, 
we use $v_x=v\cos(\theta), v_y=v\sin(\theta)$ as the decision variable. 
The decision making horizon is set to $T=100$ steps, and the control dimension is $n_a=2$. 
Despite the simplicity of the dynamic, 
the number of decision variables is $n_a\times T=200$, 
making it a challenging long horizon planning problem. 

The cost function is composed of the following terms: 
the primary term is the running cost, 
which is the distance to the goal (red disk) at each step. 
On top of that, we add the control loss term, 
plus a penalty term for not reaching the tunnel and further penalties when confronting the obstacles along the planned trajectories.
In summary, the cost function is defined as:
\begin{align}
  J =& \sum_{t=0}^{T-1} \left(\|q(t)-q_{goal}\|^2 + \gamma_v \|v(t)\|^2\right) \nonumber\\
     &+I_{\text{not in tunnel}}\left(q(T-1)\right) + \sum_{i=1}^{N_{obs}} I_{\text{obstacle}}(q_{0:T-1}),
\end{align}
where $I_{\text{not in tunnel}}$ is an indicator function that adds a large penalty if the final position is not inside the tunnel, and $I_{\text{obstacle}}$ is an indicator function that adds a penalty if the trajectory confronts any obstacle.
We set $\gamma_v=10.0$, $I_{\text{not in tunnel}}\left(q(T-1)\right)=1000$ if the agent is not inside the tunnel at the last step, $I_{obstacle}=1000\times n_c$, where $n_c$ is the number of obstacles confronted along the trajectory.
The parameters for penalty terms are chosen to ensure that circumventing the left wall and the obstacles is more cost-effective than confronting them directly, such that the optimal solution involves navigating through the tunnel and reaching the goal.
Note that the velocity is not bounded but only regularized in the cost function with $\gamma_v$.

In terms of collision handling, the left wall acts as a barrier for the agent to reach the tunnel and goal (red disk). To simplify, we restrain the agent to bounce back to its original place if the current velocity would bring the agent into contact with the wall. On the other hand, when confronting the disk barrier, we simply add a penalty to the cost function but do not restrain the dynamic accordingly (i.e., the agent can penetrate the disk obstacle but incurs a large cost).
The agent needs to make long term planning to avoid getting stuck in local minima, especially when the obstacles are placed closely. 

For each repetition of the experiment, we create the same environment for competing algorithms by randomly sampling obstacles.
We compare CBO with MPPI and CMA, where we implemented CMA according to~\eqref{eq:cma_mean_ada} and~\eqref{eq:cma_cov_ada}. We set population size $N=1000$. The covariance matrix for baselines is set to have initial diagonal value matching $\sigma=10$ for CBO, while we use an exponential decay of the $\sigma$ for CBO ensures the convergence behavior due to \Cref{remark:choice_lambda_cbo}.

The benchmark results are summarized in~\Cref{fig:benchmark_jungle_best_loss}. CBO outperforms CMA and MPPI by a large margin. In~\Cref{fig:jungle}, we visualize the trajectories generated by CBO, MPPI, and CMA, respectively. It can be observed that CBO is able to successfully navigate through the tunnel, avoiding all obstacles (in terms of obstacle avoidance, for simplicity, we consider the agent as a point mass), and reaches the tunnel, whereas MPPI and CMA get stuck in their local minima and fail to reach the tunnel.
\begin{figure}[h!]
    \centering
    \includegraphics[width=0.99\linewidth]{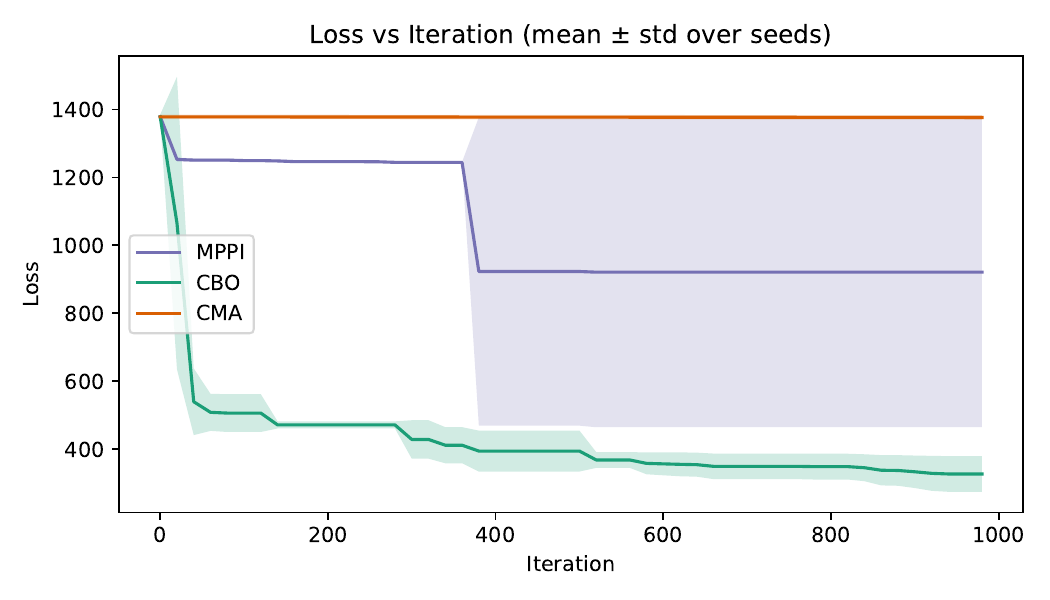}
    \vspace{-7mm}
    \caption{Long horizon planning: comparison of the best loss across the population at each iteration for the long horizon planning problem. CBO wins by a large margin. CMA implemented according to~\eqref{eq:cma_mean_ada} and~\eqref{eq:cma_cov_ada}.}\label{fig:benchmark_jungle_best_loss}
\end{figure}

\begin{figure}[h!]
    \centering
    \captionsetup[subfigure]{justification=centering} 
    \subfloat[CBO]{\includegraphics[width=0.44\linewidth]{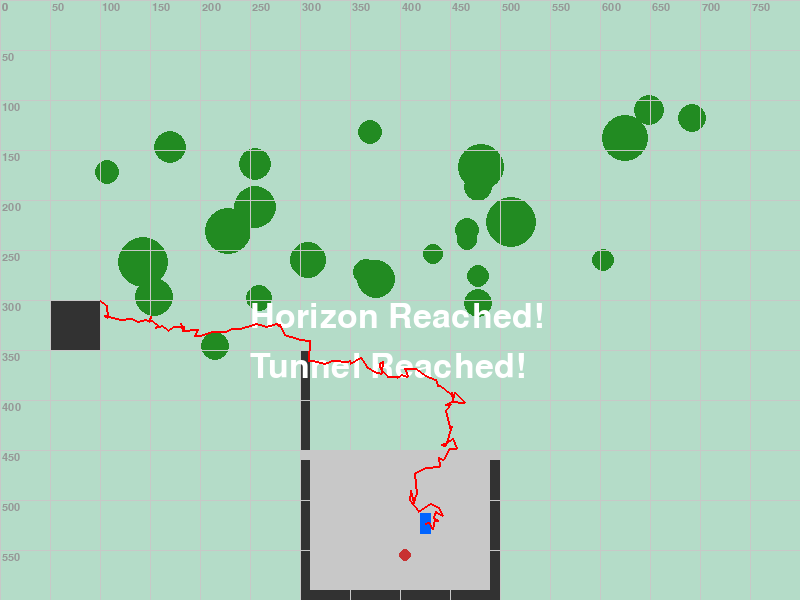}\label{fig:jungle_cbo1}}
   \hfill
   \subfloat[CBO]{\includegraphics[width=0.44\linewidth]{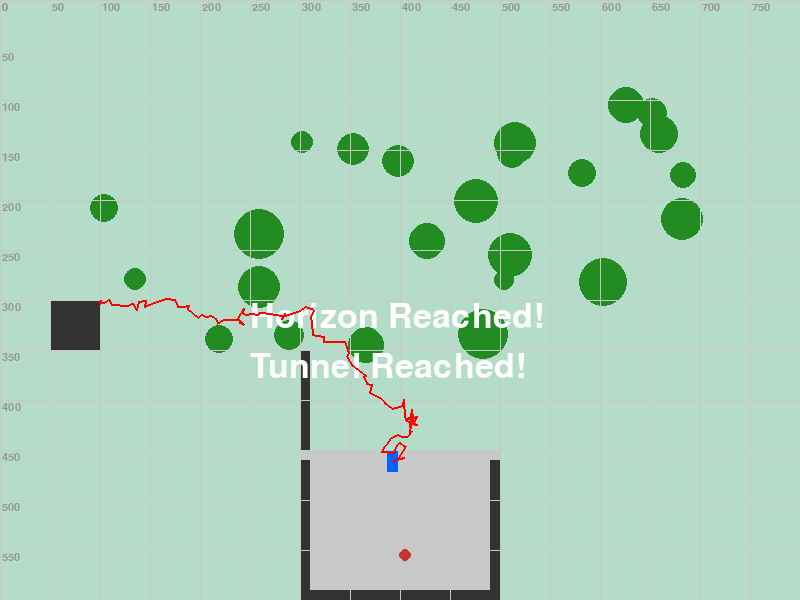}\label{fig:jungle_cbo2}}
   \vspace{1em} 
    \subfloat[MPPI]{\includegraphics[width=0.44\linewidth]{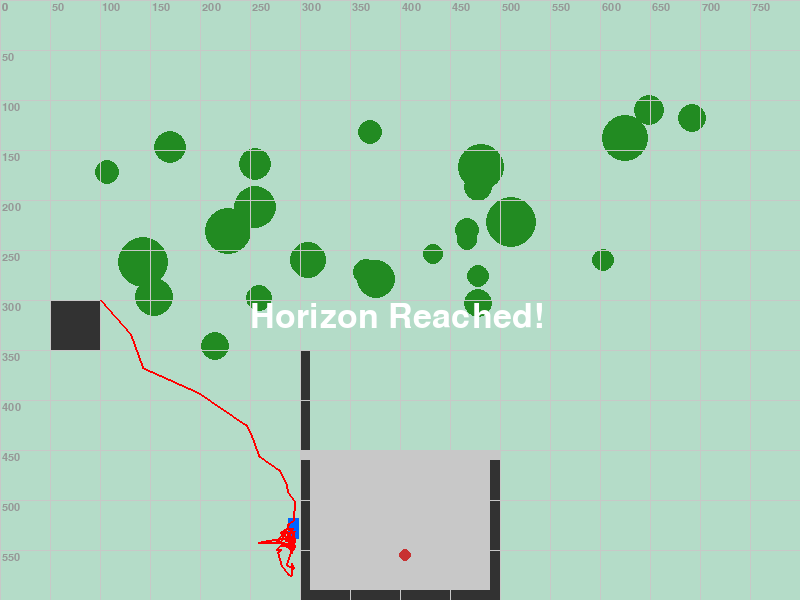}\label{fig:jungle_mppi1}}
    \hfill
    \subfloat[MPPI]{\includegraphics[width=0.44\linewidth]{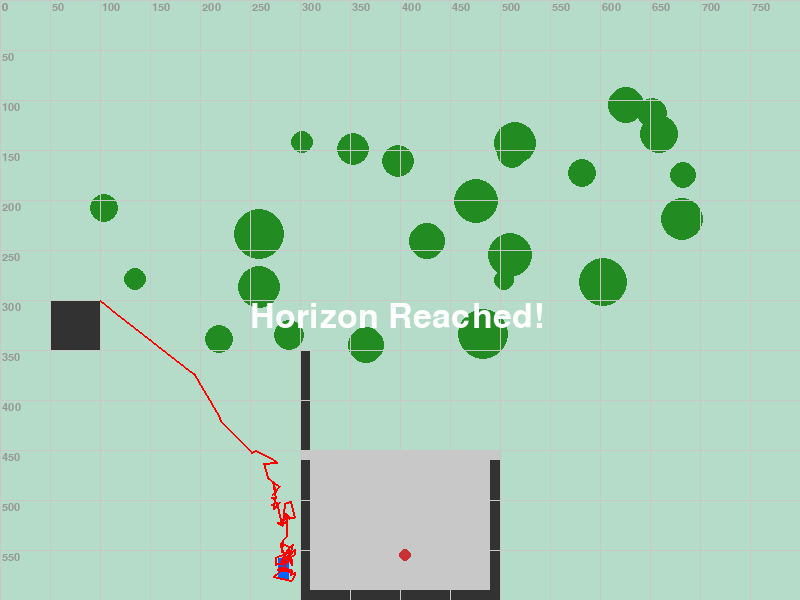}\label{fig:jungle_mppi2}}
    \vspace{1em}
     \subfloat[CMA]{\includegraphics[width=0.44\linewidth]{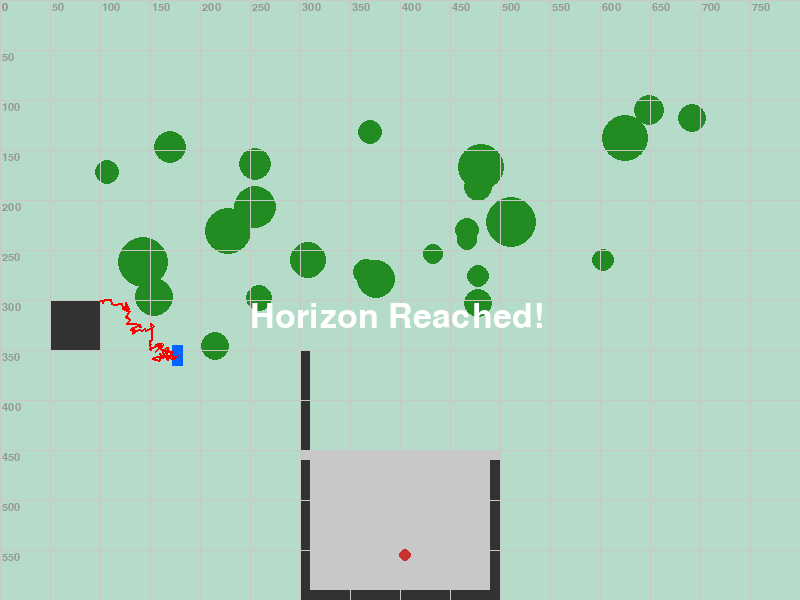}\label{fig:jungle_cma1}}
    \hfill 
    \subfloat[CMA]{\includegraphics[width=0.44\linewidth]{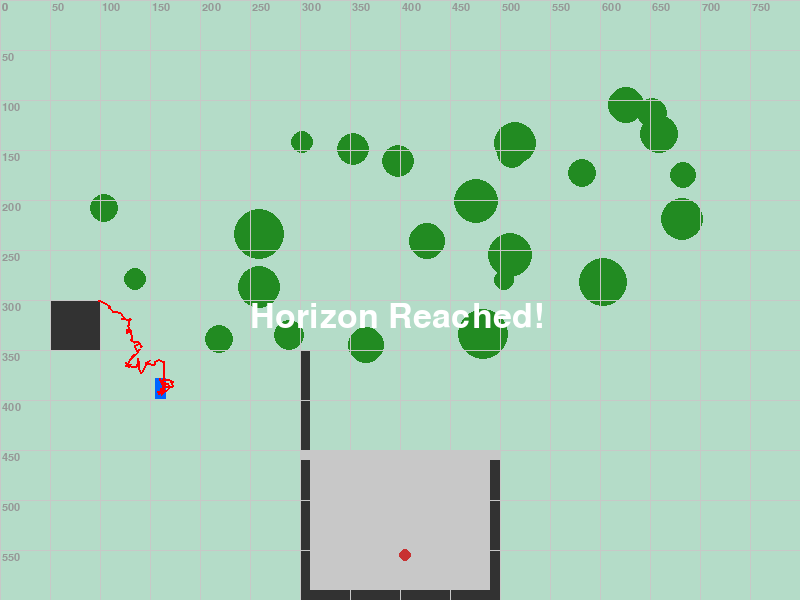}\label{fig:jungle_cma2}}
    \vspace{-2mm}
\caption{MPPI and CMA get stuck at a local minimum and fail to reach the tunnel. CBO succeeds in circumventing the left wall barriers and all obstacles (point mass agent) and reaches the tunnel. Each image shows the trajectory of a single experiment for different methods. Each column corresponds to one environment setting (spatial distribution of obstacles).}\label{fig:jungle}  
\vspace{-2mm}
\end{figure}

\begin{remark}[Why none of the algorithms reach the goal]
This scenario is extremely difficult because of the problem’s long-horizon nature and the extreme clutter in the scene.
In this task, although CBO performs significantly better than other methods, it still fails to reach the exact goal indicated by the red point inside the tunnel. The performance of CBO can be further improved by increasing the number of particles.
We refer the reader to Theorem 3.8 and Remark 3.9 in~\cite{fornasier2024consensus}, where, to each optimization problem, there is an associated rate of convergence $C_{MFA} N^{-1}$. $C_{MFA}$ characterizes the complexitiy of a problem. The corresponding $C_{MFA}$ for this task is very large, implying the need of a larger number $N$ of particles. 
\end{remark}

\subsection{Double Cartpole}\label{subsec:exp_d_cartpole}
The double cartpole (see an illustration in~\Cref{fig:double_cartpole}) has the cart that is allowed to move within the range $[-3.8m, +3.8m]$ and the pole with the cart forming a revolute joint and another pole connected to the base pole via a revolute joint. The generalized coordinates are
$\begin{bmatrix} q & \theta_1 & \theta_2 \end{bmatrix}^T$,
where $q$ is the cart position along the rail, $\theta_i$ with $i=1,2$ are the pole angles.
\begin{figure}[h!]
    \centering
\includegraphics[width=0.9\linewidth]{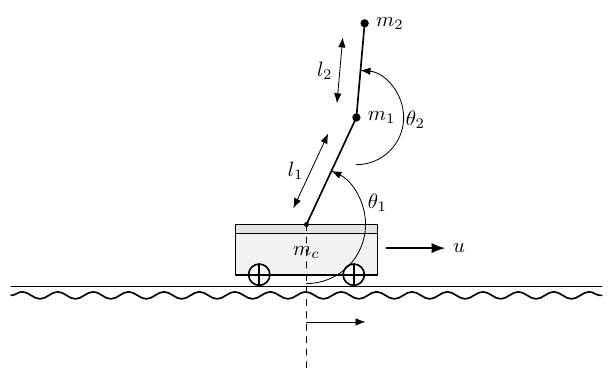}
    \vspace{-7mm}
    \caption{Double cartpole with cart mass $m_c= 1.0$ kg. Pole masses: $m_1 = m_2 = 0.1$ kg. Each pole segment has length $l_1 = l_2 = 1.0$m, giving a total pole-chain length of $2.0$m. The cart sits at 2m high in the world frame.}\label{fig:double_cartpole}
\end{figure}

The initial state is
\begin{align}
  x_0 &= \begin{bmatrix} q_0 \\ \dot{q}_0 \end{bmatrix}
  = {\begin{bmatrix} q & \theta_1 & \theta_2 & \dot{q} & \dot{\theta}_1 & \dot{\theta}_2 \end{bmatrix}^T}_{t=0}=0,
\end{align}
where $\theta_i=0$ corresponds to the pole hanging downward, and the
upright configuration is $\theta_i=\pi$.
Slider joint (cart) damping is set to $d_x = 10^{-4}$. To make the problem more challenging, we set friction loss to  $0$ for both revolute joints. The running cost is set to be the bound violation cost of control according to~\citep{Miura_Akai_Honda_Hara_2024} and the terminal cost is calculated as the summation of all the following terms with equal weighting:
\begin{itemize}
\item distance of the tip site from its upright target: $\|\mathtt{tip}_z - 4\|^2$,
  where $4\,\mathrm{m}$ is the upright tip height: the cart joint is located at height $2\,\mathrm{m}$ in the MuJoCo scene and the two $1\,\mathrm{m}$ pole segments add another $2\,\mathrm{m}$ when fully upright.
  \item $\|\mathtt{tip}_x - q\|^2$: pole tip aligned directly above the cart,   
    not leaning left or right.
  \item the distance of the cart to the center of the rail.
  \item the velocity of all degrees of freedom.
\end{itemize}
The control input is the actuator command $u$ applied to the cart, with actuator force $F=20u$, where $u\in [-u_{\max}, u_{\max}]$.
With $u_{\max}=0.5$, the nominal force range is $[-10,10]\,\mathrm{N}$.
To make the problem challenging, we deliberately set $u_{\max}$ as small as $0.5$ to increase the control difficulty, while we use a long planning horizon of $16\,\mathrm{s}$ discretized by $100$ zero-order-hold control knots. Population size $N=5000$.

Note that we do not aim for a perfect swing-up solution here; rather, the goal is to demonstrate that under such a challenging setting, CBO exhibits clear advantages over competing methods in terms of convergence and consistency.
In~\Cref{fig:benchmark_dCartpole}, we present benchmark results of loss minimization among CBO, MPPI and CMA-ES (code from~\cite{lange2023evosax}) where CBO performs the best.
CMA-ES experiences large fluctuations of cost during the iteration, while CBO consistently generates good performance.

\begin{figure}[h!]
    \centering
    \includegraphics[width=0.99\linewidth]{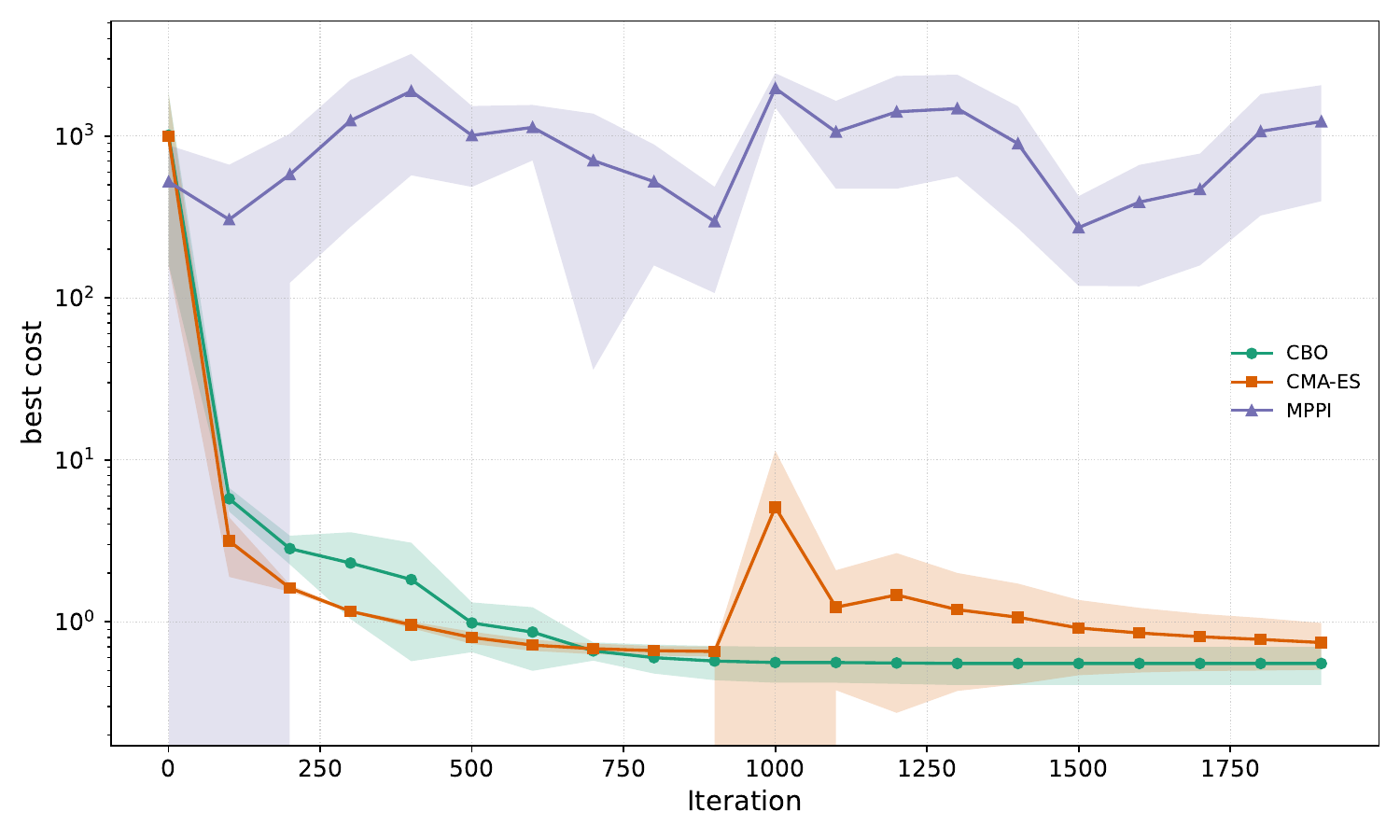}
    \vspace{-7mm}
    \caption{Benchmark double-cartpole: On average, CBO performs the best with relatively small variance. CMA-ES is from implementation~\cite{lange2023evosax}.
    }\label{fig:benchmark_dCartpole}
\end{figure}


\begin{remark}[CBO can be as fast as other zero-order methods]
The bottleneck of the computation time for zero-order methods is the simulation rollout time. Thus the double cartpole experiment, due to its relative simplicity to be simulated, is a good example to compare the running time of the different algorithms.
The statistics
in~\Cref{tab:runtime-comparison} are generated on a NVIDIA GeForce RTX 5090 Laptop GPU
with a 32 $\times$ AMD Ryzen 9-9955HX3D 16-Core processor running a Linux OS
(6.17.0-109014-Tuxedo). Here we use only $N=1000$ particles.
We report \textit{JAX} compilation time and per-iteration computation time, excluding
auxiliary code overhead. The results show no significant difference between the tested
zero-order methods in per-iteration computation time.
\end{remark}
\begin{table}[h!]
    \centering
    \caption{Runtime comparison across algorithms.}
    \vspace{-2mm}
    \label{tab:runtime-comparison}
    \begin{tabular}{lrrr}
        \hline
        Algorithm & Compile (ms) & Mean (ms) & Std (ms) \\
        \hline
        CBO & 6638.30 & 336.43 & 1.02 \\
        CMA-ES & 11369.38 & 351.28 & 12.53 \\
        MPPI & 6548.47 & 372.48 & 35.96 \\
        \hline
    \end{tabular}
\end{table}

\subsection{Humanoid}\label{subsec:exp_g1}
Our last experiment is the $23$ DoF Unitree G1 Humanoid, given $T=6$ zero-order hold knots as PD target input $q_{des}$ in~\eqref{eq:pd_target} with the following parameters:
  \begin{align}
    \tau^\ast &= k_p\,(q_{\text{des}} - q) - k_d\,\dot{q}, \label{eq:pd_target}\\
    k_p &= 500,\qquad
  k_d = 2\,\zeta\,\sqrt{k_p\,I_{\text{eff}}},\qquad
  \zeta = 1,\label{eq:pd_target_kp_kd}
  \end{align}
where $I_{\text{eff}}$ is the effective inertia of the actuator at neutral position.
The decision variable has dimension $T\times n_a = 6\times23$.
We aim to achieve \textbf{demonstration free} locomotion within $2$ seconds: Starting from a standing position, we only provide the terminal base configuration as the goal, without any running cost. 
We use the $se(3)$ difference between the starting base configuration and the goal base configuration at the last step as the cost function.


We present the benchmark results in~\Cref{fig:benchmark_best_loss_humanoid_10k} by comparing the best loss among the population of each iteration. This shows  that CBO significantly outperforms MPPI and CMA-ES (code from~\cite{lange2023evosax}). Interestingly, CBO has a large variance. Some runs are able to generate much lower cost compared to others, though the upper boundary of CBO loss variance still overrates the lower cost variance boundary of the baselines.
  \begin{figure}[h!]
    \centering
    \includegraphics[width=0.99\linewidth]{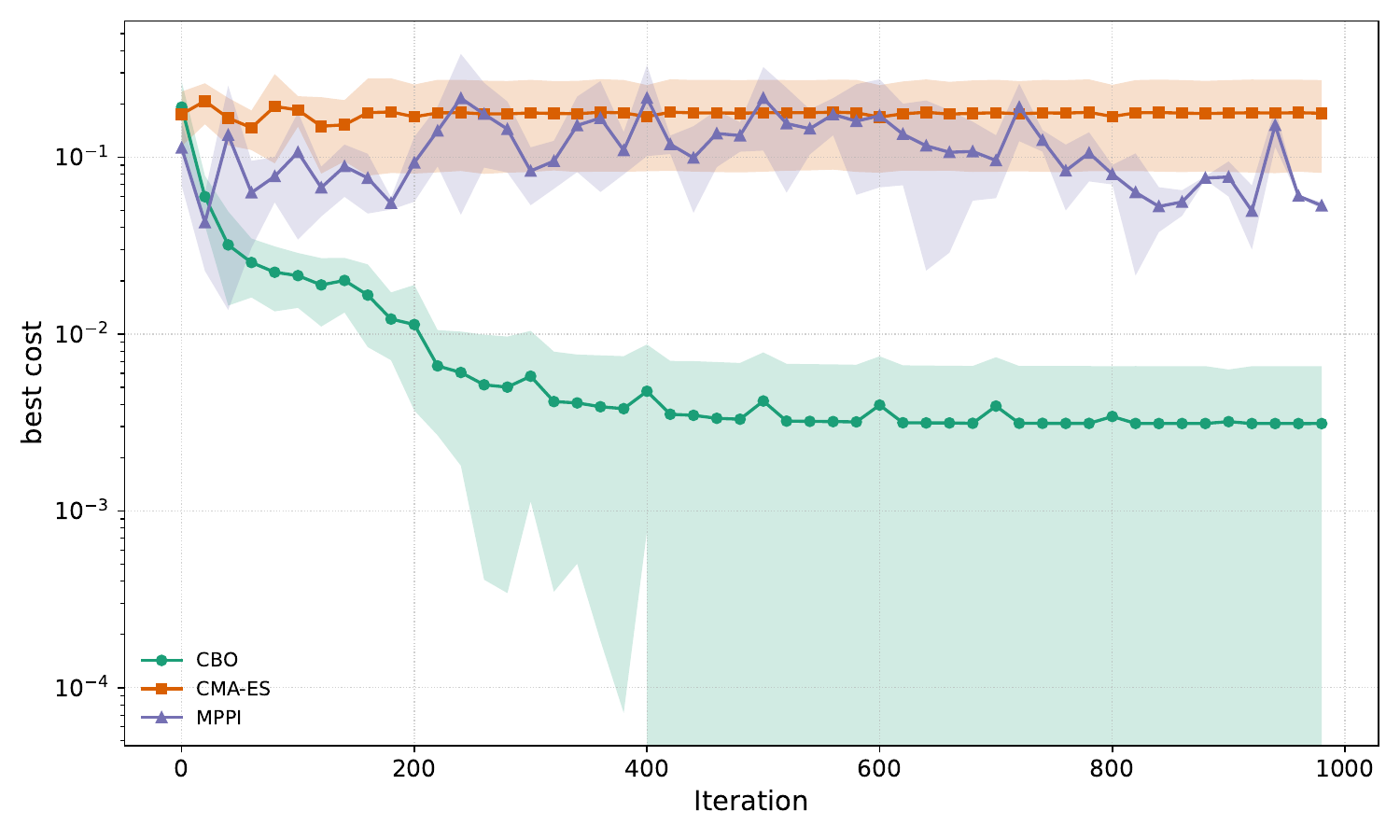}
    \vspace{-7mm}
    \caption{Comparison of the best loss among the population of each iteration for the humanoid demonstration free locomotion experiment. Population size $N=10,000$, CBO wins by large margin. CMA-ES is from implementation~\cite{lange2023evosax}.
    }
    \label{fig:benchmark_best_loss_humanoid_10k}
    \end{figure}

In \Cref{fig:snapshots_g1}, we show snapshots of the starting position, middle position, and end position of the humanoid from a trajectory optimized by CBO.
\begin{figure}[h!]
\centering
\includegraphics[width=0.32\linewidth]{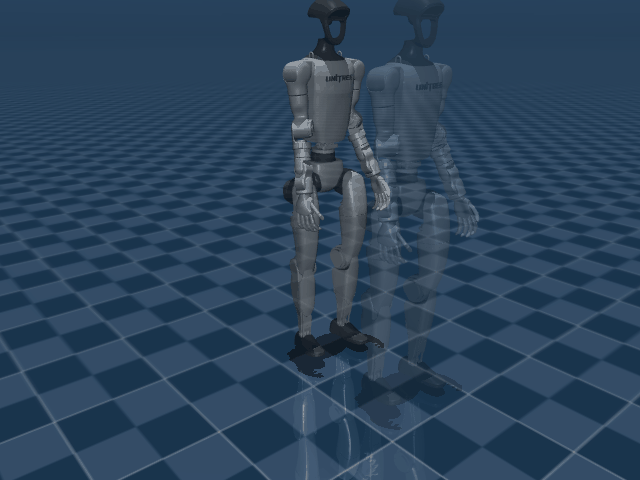}\hfill
\includegraphics[width=0.32\linewidth]{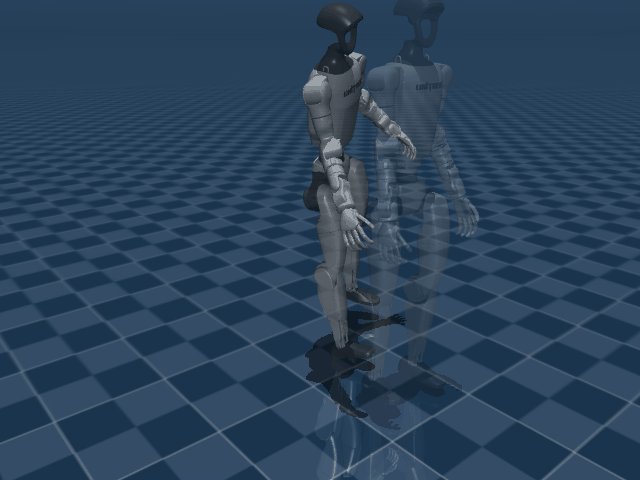}\hfill
\includegraphics[width=0.32\linewidth]{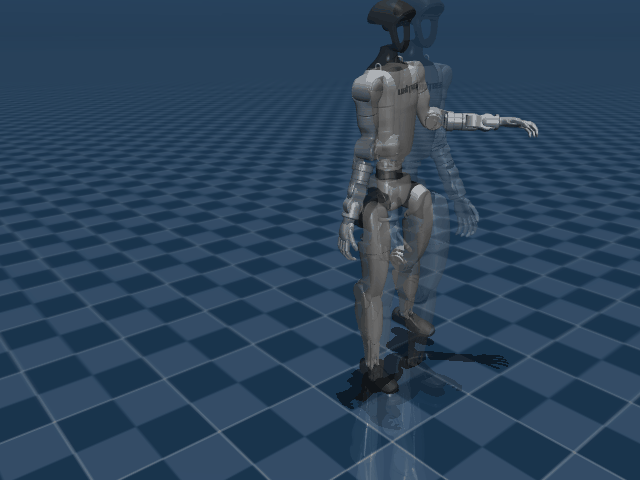}
\caption{CBO optimized trajectory: Three snapshots (start, middle, end) of the G1 humanoid trajectory in the free locomotion task for 2 seconds. The shadowed silhouette indicates the target position for the humanoid.}\label{fig:snapshots_g1}
\end{figure}




\section{Conclusion and future work}\label{sec:conclusion}
We introduced CBO to the robotics community as a step toward global trajectory optimization. By presenting a general mathematical framework for global optimization, we analyzed the limitations of widely used zero-order optimization methods in robotics and explained how CBO addresses these shortcomings. Through three challenging trajectory optimization problems, we demonstrated that CBO can find solutions with significantly lower cost for the tasks considered.
Future work includes exploring the use of CBO for discovering multi-modal behaviors, incorporating constraints, and conducting real-world robotic experiments.

\section{Acknowledgments}
Xudong Sun and Majid Khadiv acknowledge the support of the Huawei-TUM joint laboratory. Massimo Fornasier acknowledges the support of the Munich Center for Machine Learning
and the ERC Advanced Grant NEITALG, grant agreement No. 101198055.

\begin{figure}[h!]
    \centering
    \includegraphics[trim={4cm 15cm 4cm 4cm}, clip,width=0.9\linewidth]{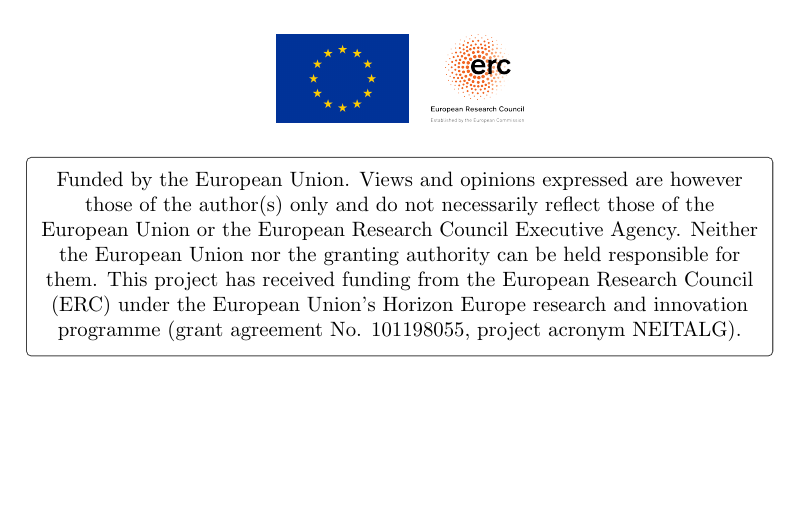}
    \label{fig:placeholder}
\end{figure}


\bibliographystyle{unsrt}
\bibliography{references}

\clearpage
\onecolumn
\appendix
To ensure consistent cross-referencing with the main text, the equation numbering is continued in this appendix.

\subsection{Proof of~\Cref{prop:fisher}.}\label{proof:fisher}

We first prove the following helper lemma. 

\begin{lemma}\label{app:fisher_information_gaussian_mean}
Let $x \in \mathbb{R}^d$ be distributed as
\[
x \sim \mathcal{N}(\mu, \Sigma),
\]
where $\mu \in \mathbb{R}^d$ is the  mean and $\Sigma \in \mathbb{R}^{d \times d}$ is the covariance matrix. Then, 
\[
\mathcal{I}(\mu)= \Sigma^{-1}.
\]
\end{lemma}

\begin{proof}
The log-likelihood for a single observation is
\[
\ell(\mu)
= \log p(x \mid \mu)
= -\frac{1}{2}{(x-\mu)}^{\top}\Sigma^{-1}(x-\mu) + C.
\]
The score function is
\[
\nabla_{\mu} \ell(\mu)
= \Sigma^{-1}(x-\mu).
\]
The Fisher Information Matrix is defined as
\[
\mathcal{I}(\mu)
= \mathbb{E}\!\left[
\nabla_{\mu} \ell(\mu)\, \nabla_{\mu} \ell{(\mu)}^{\top}
\right].
\]
Using $\mathbb{E}[(x-\mu){(x-\mu)}^{\top}] = \Sigma$, we obtain
$
\mathcal{I}(\mu)
= \Sigma^{-1}\, \Sigma\, \Sigma^{-1}
= \Sigma^{-1}.
$ 
For $n$ independent and identically distributed samples, the Fisher information becomes 
$\mathcal{I}_n(\mu) = n\, \Sigma^{-1}$.
\end{proof}
\vspace{3em}
We are now ready to prove \Cref{prop:fisher}. 
\vspace{3em}
\begin{proof}
The solution to the constrained optimization problem in~\eqref{eq:obj_natural_gradient_laplace_surrogate_plus_innerproduct_mppi} with constraint~\eqref{eq:constraint_natural_gradient_kl_sampling_distribution} can be derived using Lagrange multipliers, leading to an update rule for $\controltrajMean$ that incorporates both the gradient of the surrogate cost and the KL divergence constraint, i.e., 
\begin{align}\label{eq:lagrange}
    \mathcal{L}(\deltaMean):=\laplaceSurrogateStateControlPathCostNext +\gamma\Big(KL(\measureU(u\mid \controltrajMean+\deltaMean) |\measureU(u\mid\controltrajMean))-\beta\Big).
\end{align}

When $\measureU(u\mid \controltrajMean)$ is Gaussian, the covariance matrix remains the same, we have
\begin{align}
&KL(\measureU(u\mid \controltrajMean+\deltaMean) |\measureU(u\mid\controltrajMean))=\frac{1}{2}\|\deltaMean\|^2_{\mathcal{I}_{\controltrajMean}}= \beta,\label{eq:kl_gaussian_mean}
\end{align}
where $\mathcal{I}_{\controltrajMean}=\Sigma^{-1}$ is the Fisher information matrix of $\measureU(\controltraj)$ with respect to its mean (see~\Cref{app:fisher_information_gaussian_mean}).

This result can be derived from the closed-form expression of the KL-divergence between two multivariate Gaussian distributions below with $\Sigma_0=\Sigma_1=\Sigma$:
\begin{align}
\mathrm{KL}\!\left(
\mathcal{N}(\mu_0,\Sigma_0)
\;\|\;
\mathcal{N}(\mu_1,\Sigma_1)
\right)
=
\frac{1}{2}
\left[
\log\frac{|\Sigma_1|}{|\Sigma_0|}
- \dim(\mu)
+ \mathrm{tr}\!\left(\Sigma_1^{-1}\Sigma_0\right)
+ {(\mu_1-\mu_0)}^\top \Sigma_1^{-1} (\mu_1-\mu_0)
\right]
\end{align}

With~\eqref{eq:kl_gaussian_mean}, the gradient of the constraint is
$  \mathcal{I}_{\controltrajMean}\deltaMean$,
leading to the following stationary condition of the Lagrangian:
\begin{equation}\label{eq:stationary}
\nabla_{\deltaMean}\laplaceSurrogateStateControlPathCostNext \big|_{\deltaMean=\deltaMean^*} +\gamma \Sigma^{-1}\deltaMean^{*}=0,
\end{equation}
where we use superscript $*$ to denote the optimal solution.

Below, we compute the first term in the above equation. Recall that 
\begin{equation}
\xi(J\mid \measureU ) = -\frac{1}{\rho}\log \left(\int e^{-\rho J(\controltraj|x_0)} d\measureU(\controltraj)\right)=-\frac{1}{\rho}\log \Big(\mathbb{E}_{u\sim \measureU}[\exp\left(-\rho J(u|x_0)\right)]\Big).
\end{equation}
Therefore, we get 
\begin{align}\label{eq:grad_l}
    \nabla_{\Delta(\controltrajMean)}\laplaceSurrogateStateControlPathCostNext =-\frac{1}{\rho}
\frac{\nabla_{\Delta(\controltrajMean)} 
\mathbb{E}_{u\sim \mathbb{P}}[\exp\left(-\rho J(u|x_0)\right)]
}{\mathbb{E}_{u\sim \mathbb{P}}[\exp\left(-\rho J(u|x_0)\right)]},
\end{align}
where $\mathbb P=\measureU(\controltrajMean+\Delta(\controltrajMean))$. 

Recall that we use $\measureU (\cdot \mid \controltrajMean)$ to represent $\mathcal{N}(\controltrajMean, \Sigma)$ and thus $\mathbb P$ represents $\mathcal{N}(\controltrajMean+\Delta(\controltrajMean), \Sigma)$. We can compute the gradient in the nominator of~\eqref{eq:grad_l} using the following result. 
\begin{align}
\nabla_\theta \mathbb{E}_{x \sim p_\theta}[f(x)]
&= \nabla_\theta \int f(x)\, p_\theta(x)\, dx = \int f(x)\, p_\theta(x)\, \nabla_\theta \log p_\theta(x)\, dx \\
&= \mathbb{E}_{x \sim p_\theta}
\!\left[ f(x)\, \nabla_\theta \log p_\theta(x) \right].
\end{align}
Using the above equation and the fact that 
\begin{equation}
  \nabla_{\Delta(\controltrajMean)}\log d {\measureP}=\Sigma^{-1}(u-\controltrajMean-\Delta(\controltrajMean)),
\end{equation}
equation \eqref{eq:grad_l} becomes
\begin{align}
\nabla_{\deltaMean}\laplaceSurrogateStateControlPathCostNext=&\frac{\mathbb{E}_{u\sim \mathbb{P}}
  [\exp\left(-\rho J(u\mid x_0)\right)
\Sigma^{-1}(u-\controltrajMean-\Delta(\controltrajMean))
]}{\mathbb{E}_{u\sim \mathbb{P}}[\exp\left(-\rho J(u|x_0)\right)]}\label{eq:grad_laplace_surrogate_change_of_variable}\\
=&\Sigma^{-1}\mathbb{E}_{u\sim \measureU}\Big[\frac{
  \exp\left(-\rho J(u+\Delta(\controltrajMean)\mid x_0)\right)
(u-\controltrajMean)
}{\mathbb{E}_{u'\sim \measureU}[\exp\left(-\rho J(u'+\Delta(\controltrajMean)|x_0)\right)]}\Big]\label{eq:change_of_variable}\\
=&\Sigma^{-1}\mathbb{E}_{u\sim \measureU }[w_{\controltrajMean, u}(u-\controltrajMean)], \label{eq:grad_laplace_surrogate_final}
\end{align}
The equality in~\eqref{eq:change_of_variable} is due to change of variable, $u\leftarrow u-\Delta(\controltrajMean)$. 

Note that we use $\measureU (\cdot \mid \controltrajMean)$ to denote $\mathcal{N}(\controltrajMean, \Sigma)$ and in~\eqref{eq:grad_laplace_surrogate_final},
\begin{equation}
w_{\controltrajMean, u}=\frac{
  \exp\left(-\rho J(u+\Delta(\controltrajMean)\mid x_0)\right)
}{\mathbb{E}_{u'\sim \measureU}[\exp\left(-\rho J(u'+\Delta(\controltrajMean)|x_0)\right)]}\approx \frac{
  \exp\left(-\rho J(u)\mid x_0)\right)
}{\mathbb{E}_{u'\sim \measureU}[\exp\left(-\rho J(u')|x_0)\right)]}.\label{eq:appr_jalal}
\end{equation}
The above approximation is reasonable due to the KL-divergence constraint in~\eqref{eq:kl_gaussian_mean} with a small enough $\beta$, i.e.~the shift between the two distributions being constrained to be minimal. 

Putting together~\eqref{eq:stationary} and \eqref{eq:grad_laplace_surrogate_final} results in
\begin{align}
    \deltaMean^{*}&=-\frac{1}{\gamma}\mathbb{E}_{u\sim \measureU }[w_{\controltrajMean, u}(u-\controltrajMean)].
\end{align}


Thus following natural gradient with rate $\gamma$, i.e., \eqref{eq:grad_fisher_kl_laplace}, we obtain the following update rule:
\begin{align}
\controltrajMean^{(r+1)}&= \controltrajMean^{(r)}-\gamma\deltaMean^* \\
                        &=\controltrajMean^{(r)}
                        +\mathbb{E}_{u\sim \measureU}w_{\controltrajMean, u}(u-\controltrajMean)\\
                        &=\mathbb{E}_{u\sim \measureU}w_{\controltrajMean, u}u.
\end{align}
The last equality comes from the fact that $\mathbb{E}_{u\sim \measureU}[w_{\controltrajMean, u}\controltrajMean]=\controltrajMean$

\end{proof}






\subsection{Proof of~\Cref{prop_cbo}.}\label{proof:cbo}
To prove the conclusion, we provide the following lemmas based on reformulations of what was established in
~\cite{fornasier2024consensus}.
\vspace{1cm}
\begin{lemma}[Exponential decay of Lyapunv function in mean  dynamic]\label{lemma:lyapunov_exp_decay}
Define $\mathcal{V}=\int{\frac{1}{2}\|\controltraj-\controltraj^{*}\|_2^2}d\measureU(\controltraj)$ (Lyapunov function).
 When~\Crefrange{eq:lipschitz}{eq:cbo_ball} holds, and $\controltraj^{*}$ lives in the support of $\measureU^{(0)}$, given $0<\vartheta<1$,
$\exists\rho(\vartheta)$ in~\Cref{eq:PI_weight} which is large enough, such that  
\begin{equation}
  \mathcal{V}(\measureU^{(r)})\le \mathcal{V}(\measureU^{(0)})\exp\left(-r(1-\vartheta)(2\lambda-T\times n_a \sigma^2)\right) 
\end{equation}
under the mean field dynamic (see Equation(7-8) in~\cite{fornasier2024consensus}).
\end{lemma}
\begin{proof}
 This is a restatement of~\textit{Theorem 3.7} of~\cite{fornasier2024consensus},    
\end{proof}
\vspace{1cm}
\begin{lemma}[Convergence in probability: Wasserstein distance to global optimizer]\label{lem:global_opt_cbo}
When~\Crefrange{eq:lipschitz}{eq:cbo_ball} hold, given $0<\vartheta<1$,
$\exists\rho(\vartheta)$ in~\eqref{eq:PI_weight} which is large enough and iteration $r>r^{*}(\underline{\mathcal{V}})>0$,
\begin{align}
&\Pr\left(
  \|\frac{1}{N}\sum_1^N \controltraj^{(r,i)}- \controltraj^{*}\|_2^2 \le \epsilon_e
\right) \ge 1-\delta\left(\epsilon_e, N, \Delta r, n_a \times T, \underline{\mathcal{V}})\right)\label{eq:prob_particles_dist2global_optimizer}
\end{align}
where
\begin{align}
r^{*}&=\frac{1}{1-\vartheta} \frac{1}{2\lambda-n_a\times T \sigma^2}\log\frac{\mathcal{V}(\measureU^{(0)})}{\underline{\mathcal{V}}} \label{eq:iteration_steps_decay_selection1}
\end{align}
and 
\begin{align}
\mathcal{V}(\measureU)=\frac{1}{2}\int{\|\measureU-\measureU^{*}\|}^2 d\measureU(\controltraj),\quad 
0<\underline{\mathcal{V}}<\mathcal{V}(\measureU^{(0)}).
\end{align}

In~\eqref{eq:prob_particles_dist2global_optimizer}, $\delta\left(\epsilon_e, N, \Delta r, n_a \times T, \underline{\mathcal{V}})\right)$ decreases with larger population size $N$ and smaller $\Delta r$ (Euler-Maruyama time interval in~\eqref{eq:cbo_dyn_discrete}) and smaller $\underline{V}$ (corresponding to more Euler-Maruyama integration steps in~\eqref{eq:cbo_dyn_discrete}).

\end{lemma}
\begin{proof}
This statement is a reformulation of Theorem 3.8 (convergence of CBO under finite sample size approximation to the Fokker-Planck equation) in~\citep{fornasier2024consensus}.
\end{proof}

\vspace{1cm}
Now we present the proof of~\Cref{prop_cbo}.
\vspace{1cm}
\begin{proof}
To simplify notations, we use $\measureU^{(r^{*})}(\controltraj)$ to replace
$\measureU^{(r^{*})}(\controltraj\mid \controltrajMean^{(r^{*})})$.

We first prove the mean-field approximation~\cite{fornasier2024consensus} case:

According to~\Cref{lemma:lyapunov_exp_decay},
under assumption~\Crefrange{eq:lipschitz}{eq:cbo_ball}, the diffusion dynamic in~\eqref{eq:cbo_dyn} leads to descent of $\mathcal{V}$ exponentially.

Choose the first phase duration $r_1$ sufficiently large, such that $\mathcal{V}$ is small enough after the particles enter the neighborhood $B_{\kappa}(\controltraj^{*})$ of the global optima $\controltraj^{*}$.

In the second phase, the particles are close enough to the global optima, i.e., they are within the neighborhood $B_{\kappa}(\controltraj^{*})$. 
 Once inside this neighborhood, define $L=\sup_{B_{\kappa}(\controltraj^{*})} L(\controltraj)$, which results in
  \begin{align}
    \xi(J\mid \measureU(\controltraj))&=\int{\left(J(\controltraj\mid x_0)-J(\controltraj^{*}\mid x_0)\right)}d\measureU(\controltraj)~(\text{assume}~J(\controltraj^{*}\mid x_0)=0)\label{eq:}\\
                                      &\le \int{L}{\|\controltraj-\controltraj^{*}\|_2}d\measureU(\controltraj)\label{eq:coercivity}~(\text{Lipschitz})\\
                                      &\le L\sqrt{\int{\|\controltraj-\controltraj^{*}\|_2^2}d\measureU(\controltraj)}~(\text{Cauchy-Schwarz inequality})\label{eq}\\
                                      &=L \sqrt{2\mathcal{V}(\measureU)}\label{eq:loss_integral_lt_lyapunov_any_iter}
\end{align}
Since $\mathcal{V}(\measureU)$ continues to decrease exponentially under the diffusion dynamic in~\eqref{eq:cbo_dyn} according to~\Cref{lemma:lyapunov_exp_decay}, $\xi(J\mid \measureU(\controltraj))$ also decreases accordingly. 
Thus, we have
 \begin{align}
  &\int J(\controltraj|x_0)d\measureU^{(r^{*})}(\controltraj)\label{eq:loss_integral}\\
  \le&L \sqrt{2\mathcal{V}(\measureU^{(r^*)})}~(\eqref{eq:loss_integral_lt_lyapunov_any_iter})\label{eq:loss_integral_lt_lyapunov}\\
  \le& L\sqrt{2\mathcal{V}(\measureU^{(0)})\exp\left(-r^{*}(1-\vartheta)(2\lambda-T\times n_a \sigma^2)\right)}~(\text{exponential decay of}~\mathcal{V}(\measureU),\Cref{lemma:lyapunov_exp_decay})\label{eq:lyapunov_exp_decay}\\
  =&L\sqrt{2\exp\left(-r^{*}(1-\vartheta)(2\lambda-T\times n_a \sigma^2)\right)}\sqrt{\int{\|\controltraj-\controltraj^{*}\|_2^2}d\measureU^{(0)}(\controltraj)}\\
  \le&L\frac{\sqrt{2\exp\left(-r^{*}(1-\vartheta)(2\lambda-T\times n_a \sigma^2)\right)}}{\eta}\sqrt{\int{\left(J(u)-J(u^{*})\right)^2}d\measureU^{(0)}(\controltraj)}~(\text{inverse continuity in}~\eqref{eq:assumption_inverse_continuity})\\
  =&L\frac{\sqrt{2\exp\left(-r^{*}(1-\vartheta)(2\lambda-T\times n_a \sigma^2)\right)}}{\eta}\sqrt{\int{\left(J(u)\right)^2}d\measureU^{(0)}(\controltraj)}~(\text{using}~J(u^{*})=0)\\
  =&\frac{L\sqrt{2\exp\left(-r^{*}(1-\vartheta)(2\lambda-T\times n_a \sigma^2)\right)}}{\eta}\zeta(J, \measureU^{(0)})\int J(\controltraj|x_0)d\measureU^{(0)}(\controltraj)
  \end{align}
where in the last equation we define $\zeta(J, \measureU^{(0)})$,~s.t.
\begin{align}
   \zeta(J, \measureU^{(0)})\int J(\controltraj|x_0)d\measureU^{(0)}(\controltraj) 
   &=\sqrt{\int{{\left(J(\controltraj\mid x_0)\right)}^2}d\measureU^{(0)}(\controltraj)}
\end{align}

Due to the exponential  in~\eqref{eq:lyapunov_exp_decay}, we can always let the iteration persist until we find a big enough $r^{*}$ such that $\frac{L\sqrt{2\exp\left(-r^{*}(1-\vartheta)(2\lambda-T\times n_a \sigma^2)\right)}}{\eta}\zeta(J, \measureU^{(0)})<1$,  such that
\begin{align}
 \int J(\controltraj|x_0)d\measureU^{(r^{*})}(\controltraj)   < &\int J(\controltraj|x_0)d\measureU^{(0)}(\controltraj)\label{eq:j_integral_decay}    
\end{align}
  
In the finite sample approximation, a similar argument holds according to~\Cref{lem:global_opt_cbo}:  Choose $\epsilon_e$ small enough in ~\eqref{eq:prob_particles_dist2global_optimizer} such that with high probability (at the cost of large population size $N$, smaller $\Delta r$ and long iterations), all particles reside in $B_{\kappa}(u^{*})$. Then~\eqref{eq:j_integral_decay} also holds following the same arguments as in the mean-field  (via choosing a even smaller $\epsilon_e$).
The target consensus point can be regarded as a convex combination of particles (see~\eqref{eq:cbo_mean_particle_dirac}): when all particles converge to the global optimizer, the convex combination converges as well. 

\end{proof}

\begin{remark}
The conclusion in \Cref{prop_cbo} is essential, as it shows that, compared to the other zero-order optimization methods considered in this paper, CBO generates a population with significantly lower average cost. While alternative methods typically produce a large number of poor-quality samples alongside a single high-performing solution, CBO yields a population whose members are, on average, well behaved and consistently low cost.
\end{remark}

\subsection{Details about the illustrative objective function in~\Cref{fig:finite_locals_unique_global}}\label{app:subsec:himmelblau}
\paragraph{Objective function.}

We consider the following non-convex function:
\begin{equation}
\label{eq:himmelblau_objective}
f(x,y)=
  \underbrace{\bigl(x^2 + y - 11\bigr)^2 + \bigl(x + y^2 - 7\bigr)^2}_{h(x,y)}
  +\alpha\bigl((x-3)^2 + (y-2)^2\bigr).
\end{equation}
where $\alpha>0$ is a small constant.

The first term $h(x,y)$ is the classical Himmelblau function, which is non-convex and admits a finite number of isolated minimizers. The second term is a quadratic penalty centered at $(3,2)$ that is zero at this point and strictly positive elsewhere.

\paragraph{Global minimizer.}
Since $h(x,y)\ge 0$ for all $(x,y)$ and
\[
h(3,2)=0,
\]
it follows that
\[
f(3,2)=0,
\]
and therefore $(3,2)$ is a global minimizer of $f$. Moreover, for any $(x,y)\neq(3,2)$, the penalty term satisfies
\[
\alpha\bigl((x-3)^2 + (y-2)^2\bigr) > 0,
\]
which implies $f(x,y)>0$. Hence, $(3,2)$ is the \emph{unique} global minimizer of $f$.

\paragraph{Local minimizers.}
The Himmelblau function $h(x,y)$ has four isolated minimizers located at $(3,2)$, $(-2.805118,\,3.131312)$, $(-3.779310,\,-3.283186)$, and $(3.584428,\,-1.848126)$, all of which achieve the same minimal value of zero for $h$.

After adding the quadratic penalty, the point $(3,2)$ remains unchanged and becomes the unique global minimizer of $f$, while the remaining three minimizers persist as \emph{strict local minimizers} with strictly larger objective values. Since the perturbation is smooth and $\alpha$ is small, these local minimizers remain isolated and lie in neighborhoods of the corresponding minimizers of $h$.

\paragraph{Summary.}
The function $f(x,y)$ is smooth and non-convex, admits a \emph{finite number of local minimizers}, and possesses a \emph{single global minimizer} at $(3,2)$. This structure makes it a convenient test problem for studying optimization dynamics in the presence of multiple local minima.

\subsection{Empirical comparision of drift (decay) rate $\lambda$ for the double cartpole experiment}
To illustrate the effects of $\lambda$ in~\Cref{remark:choice_lambda_cbo}, we used a simpler setting of the double cartpole experiment and plot the cost curve below in~\Cref{fig:sweep_lambda}.
\begin{figure}[h!]
\centering
\includegraphics[width=0.7\linewidth]{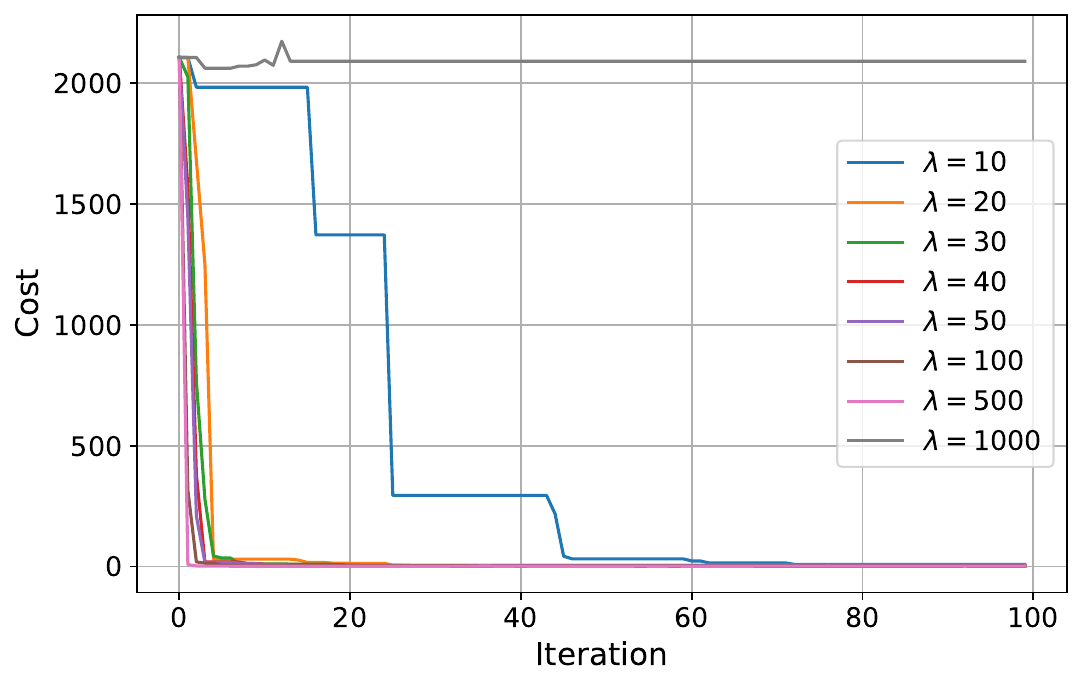}
\vspace{-3mm}
\caption{Too large $\lambda$ hurts performance}\label{fig:sweep_lambda}
\end{figure}\vspace{-5mm}
From the plot, we illustrate that too large $\lambda$ can hurt the performance of the algorithm.

\end{document}